%% file: arxiv-main.tex
\documentclass[11pt]{article}
\usepackage[utf8]{inputenc} 
\usepackage{geometry}
\usepackage{babel}
\usepackage{array}
\usepackage{verbatim}
\usepackage{diagbox}
\usepackage{enumitem}

\usepackage{amsmath, amsfonts, amsthm, bm,  amssymb, longtable, multirow, mathrsfs}
\usepackage{epstopdf, graphicx, subcaption,dsfont}
\usepackage{algorithm}  
\usepackage{algorithmicx}  
\usepackage{algpseudocode}
\usepackage{booktabs}
\usepackage{array}
\usepackage{booktabs}
\usepackage[numbers]{natbib}

\usepackage{setspace}

\usepackage{xcolor}

\definecolor{darkblue}{rgb}{0,0,.7}
\RequirePackage[colorlinks,allcolors=darkblue]{hyperref}

\usepackage[capitalize]{cleveref}
\usepackage{crossreftools}

\allowdisplaybreaks

\theoremstyle{plain}
\newtheorem{theorem}{Theorem}

\newtheorem{lemma}{Lemma}[section]

\theoremstyle{definition}

\newtheorem{assumption}{Assumption}
\renewcommand{\theassumption}{A\arabic{assumption}}

\theoremstyle{remark}

\usepackage{listings}
\definecolor{cc}{RGB}{0,0,255}
\definecolor{rck}{RGB}{255,0,0}

\input{command.tex}

\begin{document}

\begin{center}
  \Large{Memorize to Generalize: on the Necessity of Interpolation in\\
    High Dimensional Linear Regression} \\
  \vspace{.5cm}

  \large{Chen Cheng$^1$ ~~~~ John Duchi$^{1,2}$ ~~~~ Rohith Kuditipudi$^3$} \\
  \vspace{.25cm}
  \large{Departments of $^1$Statistics, $^2$Electrical Engineering,
    and $^3$Computer Science \\
    Stanford University}
  \\
  \vspace{.2cm}
  \large{February 2022; \quad Revised June 2022}
\end{center}



\input{sections/abstract.tex}
\input{sections/introduction.tex}

\input{sections/main-results.tex}

\input{sections/proof-isotropic.tex}
\input{sections/discussion.tex}

\bibliography{bib}
\bibliographystyle{abbrvnat}

\newpage
\appendix

\input{sections/matrix-review.tex}
\input{sections/proof-technical.tex}
\input{sections/proof-isotropic-additional.tex}

\input{sections/proof-isotropic-interpolator.tex}

\input{sections/proof-anisotropic.tex}
\input{sections/proof-hilbert-duality.tex}

\end{document}

%% file: command.tex
\newcommand{\opnorm}[1]{{\left\vert\kern-0.25ex\left\vert\kern-0.25ex\left\vert #1 
		\right\vert\kern-0.25ex\right\vert\kern-0.25ex\right\vert}}

\newcommand{\cas}{\stackrel{a.s.}{\rightarrow}}

\newcommand{\normal}{\mathsf{N}}
\newcommand{\train}{\mathsf{Train}}
\newcommand{\pred}{\mathsf{Pred}}
\newcommand{\cost}{\mathsf{Cost}}

\newcommand{\var}{\mathsf{Var}}

\newcommand{\wb}[1]{\overline{#1}}
\newcommand{\ind}{\mathds{1}}
\newcommand{\Ep}{\mathbb{E}}
\newcommand{\E}{\mathbb{E}}

\newcommand{\real}{\mathbb{R}} 
\newcommand{\R}{\mathbb{R}} 
\newcommand{\C}{\mathbb{C}} 
\newcommand{\prn}[1]{\left({#1}\right)} 
\newcommand{\brk}[1]{\left[{#1}\right]} 
\newcommand{\brc}[1]{\left\{{#1}\right\}} 
\newcommand{\norm}[1]{\left\|{#1}\right\|} 
\newcommand{\ltwo}[1]{\norm{#1}_2}  
\newcommand{\ltwobig}[1]{\normbig{#1}_2}  
\newcommand{\normbig}[1]{\big\|{#1}\big\|} 
\newcommand{\normtwo}[1]{\norm{#1}_2}

\newcommand{\est}[1]{\widehat{#1}}

\newcommand{\half}{\frac{1}{2}}

\newcommand{\mc}[1]{\mathcal{#1}}

\newcommand{\minimize}{\mathop{\textup{minimize}}}
\newcommand{\subjectto}{\mathop{\textup{subject~to}}}

\newcommand{\Tr}{\mathsf{Tr}}
\newcommand{\defeq}{:=}

\definecolor{innerboxcolor}{rgb}{.9,.95,1}
\definecolor{outerlinecolor}{rgb}{.6,0,.2}

\newcommand{\openright}[2]{\left[{#1},{#2}\right)}

\newcommand{\epsdeformed}{\epsilon_{\sigma,\textup{def}}}
\newcommand{\rhodeformed}{\rho_{\textup{def}}}
\newcommand{\rhools}{\rho_{\textup{ols}}}
\newcommand{\epsols}{\epsilon_{\sigma,\textup{ols}}}
\newcommand{\simiid}{\stackrel{\textup{iid}}{\sim}}

\newcommand{\thetaols}{\est{\theta}_{\mathsf{ols}}}
\newcommand{\Aols}{A_{\mathsf{ols}}}

\newcommand{\Hlin}{\mathcal{H}_{\textup{lin}}}
\newcommand{\Hsq}{\mathcal{H}_{\textup{sq}}}

%% file: sections/abstract.tex
\begin{abstract}
  We examine the necessity of interpolation in overparameterized
  models, that is, when achieving optimal predictive risk
  in machine learning problems requires (nearly)
  interpolating the training data. In particular, we consider simple
  overparameterized linear regression $y = X \theta + w$
  with random design $X \in
  \real^{n \times d}$ under the proportional asymptotics $d/n \to \gamma
  \in (1, \infty)$.  We precisely characterize how prediction (test) error
  necessarily scales with training error in this setting.  An
  implication of this characterization is that as the label noise variance
  $\sigma^2 \to 0$, any estimator that incurs at least $\mathsf{c}\sigma^4$
  training error for some constant $\mathsf{c}$ is necessarily suboptimal
  and  will suffer growth in excess prediction error at least
  linear in the training error. Thus, optimal performance requires fitting
  training data to substantially higher accuracy than the inherent noise
  floor of the problem.
\end{abstract}

%% file: sections/introduction.tex
\section{Introduction}

Conventional machine learning wisdom \citep[e.g.][]{VapnikCh71} posits that
the
size of a model's training data must be large relative to its effective
capacity---for which parameter count often serves as a proxy---in order for
the model to have good generalization.  Yet despite the fact that many
common families of modern machine learning models (e.g., deep neural
networks) are overparameterized in the sense that they are demonstrably able
to interpolate arbitrary relabelings of their training data, they tend to
generalize remarkably well in practice even after optimizing the empirical
risk to zero~\citep{ZhangBeHaReVi17}.

This \textit{benign overfitting} phenomenon has spurred considerable
recent interest and effort within the learning theory community toward
understanding learning in the overparameterized regime, where the empirical
risk minimizer is underdetermined \citep{BelkinHsMi18, BelkinMaMa18,
  BelkinRaTs19, HastieMoRoTi19, MuthukumarVoSa19, BartlettLoLuTs20,
  BelkinHsXu20, LiangRa20, MeiMo21}. Yet while overparameterized
interpolating models evidently generalize well, both in theory and practice, there
nonetheless remains at least some reason to be skeptical of the notion that
interpolation is necessarily ``benign.''  Indeed, numerous desiderata beyond
prediction risk---for example, privacy and security concerns---motivate an
explicit preference for models that do not interpolate, or in particular,
memorize, their training data.  An alternative and perhaps less auspicious
explanation for benign overfitting is that many of the crowdsourced
benchmarks the machine learning community uses to evaluate models, such as
ImageNet~\citep{DengDoSoLiLiFe09}, have limited label uncertainty: examples
with high annotator disagreement are in many cases explicitly
withheld~\citep{DengDoSoLiLiFe09,RechtRoScSh19}, mitigating the danger of
overfitting to label noise.


Thus, while interpolation may \textit{suffice} to learn models with strong
generalization, it is natural to wonder whether interpolation---or more
evocatively, memorization---is \textit{necessary} for learning in the
overparameterized regime. Here we take a phenomenological approach,
developing a simple model to explicate and predict behavior of statistical
learning procedures, and motivated by the question of the necessity
of memorization, we precisely characterize how prediction risk must scale with empirical risk.
Considering a simple linear model $y = x^\top \theta + w$,
we define memorization in terms of the empirical risk, and formulate
the cost of not fitting the training data as an optimization problem over a
class of estimators $\mathcal{H}$,
\begin{equation} \label{eq:intro-problem}
  \begin{aligned}
    \minimize_{\est{\theta} \in \mathcal{H}} & \qquad \pred \prn{\est{\theta}}
    \defeq \Ep\big[(x^\top \est{\theta} - y)^2 \mid X\big]
    \\
    \subjectto & \qquad \train\prn{\est{\theta}}
    \defeq \frac{1}{n} \Ep\left[\ltwobig{X\est{\theta} - Y}^2 \mid X\right]
    \geq \epsilon^2 \, ,
  \end{aligned}
\end{equation}
where the expectations in $\pred$ and $\train$ are taken conditional
on the over the training
data $Y$ defining $\est{\theta}$ conditional on $X$,
as well as the future data point $(x,
y)$, so that $\pred \prn{\cdot}$ and $\train\prn{\cdot}$ denote the expected prediction and training error given a prior over the true model
parameter $\theta$ (respectively).

We take as inspiration the recent line of work~\citep{Feldman20,
  BrownBuFeSmTa21}, which gives scenarios in which certain formal notions of
memorization are necessary for a model to generalize well. We build on this
by studying the extent to which memorization remains necessary even in the
simplest settings: random design linear regression with independent
noise. For our initial analysis, we assume the estimator $\est{\theta}$ is
linear in $y$, which includes least-norm interpolants and ridge regression
as special cases. Here, we obtain a tight asymptotic characterization of the
optimal solution to the problem~\eqref{eq:intro-problem} (see
Theorems~\ref{thm:cost-isotropic} and~\ref{thm:cost-general-cov}). Key to
our analysis is to show that, even though problem~\eqref{eq:intro-problem}
is non-convex, strong duality obtains, and then leverage tools from random matrix theory to obtain analytic
formulae for the optimal prediction risk by integrating over the spectrum of
the empirical data covariance.  We find that memorization of label noise is
in fact necessary for generalization even in the simple case of linear
regression; in particular, the threshold $\epsilon^2$ above which the
optimal prediction risk is no longer achievable tends to zero asymptotically
faster than the variance of the label noise---so we must fit linear
regression models to (training) accuracy substantially better than the
intrinsic noise floor of the problem. Beyond this threshold the excess
prediction risk grows linearly with the empirical risk.  Finally, assuming
Gaussian noise $w$ and a Gaussian prior over $\theta$ in
problem~\eqref{eq:intro-problem}, we extend our analysis to hold not only
for linear estimators, but for general $\mathcal{H}$ comprised of all
square-integrable estimators (see
Theorem~\ref{thm:strong-duality-hilbert}), meaning that our characterization
holds for (essentially) any estimator.


\subsection{Related work}
\label{section:related-work}

Neither interpolation nor memorization of training data is a new phenomenon
in machine learning. Classical algorithms, such as $k$-nearest
neighbors and (kernel) support vector machines, explicitly encode the
training data into the learned model. Some explicitly interpolate training
data and still enjoy performance guarantees; for example, the
$1$-nearest neighbor algorithm interpolates its training data and has
classification risk at most twice the Bayes' error~\citep{CoverHa67}.

Nonetheless, the success of deep learning has spurred renewed interest in
interpolating models. Recent work has sought to develop an understanding of
``implicit regularization'': whereas most minimizers of the empirical risk
may generalize poorly, standard learning algorithms used in practice such as
(stochastic) gradient descent tend to converge to solutions that do
generalize well, even in the absence of explicit regularization terms in the
training objective~\citep{GunasekarWoBhNeSr17, SoudryHoNaGuSr18,
  GunasekarLeSoSr18a, AroraCoHuLu19, AroraDuHuLiWang19, JiTe19}.  In the
particular case of overparameterized linear regression, gradient descent
initialized at the origin trivially recovers the ordinary least-squares
(OLS) estimator, which in overparameterized settings is the minimum norm
interpolant.  Most relevant to our work, \citet{HastieMoRoTi19} give
formulae for the asymptotic error of ridge-type estimators, including the
minimum norm interpolant, as the number of features $d$ and training
observations $n$ tend to infinity in the proportional regime where $d/n \to
\gamma > 1$ for both isotropic and anisotropic
features. \citet{MuthukumarVoSa19} give corresponding non-asymptotic lower
bounds, with matching upper bounds for certain particular feature
distributions, on the minimal error achievable among all interpolating
solutions. \citet{BartlettLoLuTs20} consider regression over general Hilbert
spaces, showing that the minimum norm interpolant achieves optimal error
assuming certain conditions on the effective rank of the feature covariance.
Our results complement this line of work: not only can overparameterized
interpolating models generalize well,
but in fact interpolation is \textit{necessary} to achieve good
generalization.

Our work pursues a line of inquiry \citet{Feldman20} originates, which
studies memorization in the setting of multi-class classification, where the
data distribution is a heavy-tailed mixture over a finite set of
subpopulations.  He defines memorization in terms of the sensitivity of a
model's predictions to the inclusion or exclusion of a particular
observation in its training data, and under the assumption that the class
labelings of distinct subpopulations are essentially independent---i.e., an
observation drawn from one subpopulation yields limited to no information
about the labels of the other subpopulations---proves that memorization is
necessary to achieve optimal generalization.
\citet{BrownBuFeSmTa21} extend these results, which are specific to label
memorization, to incorporate an information-theoretic notion of memorizing
the input observations in carefully constructed combinatorial settings,
including next-symbol prediction and clustering on the hypercube.  In
contrast, we attempt a simpler tack: ordinary linear regression with
standard distributional assumptions, construing memorization strictly in
terms of training error.




%% file: sections/main-results.tex
\newcommand*{\horzbar}{\rule[.5ex]{2.5ex}{0.5pt}}

\section{Problem formulation}
Given a design matrix $X = \mathbb{R}^{n \times d}$ ($d \geq n$), an unknown signal $\theta \in \mathbb{R}^d$ and a noise vector $w$ such that $\Ep [w] = 0$ and $\var(w) = \sigma^2 I_n$, consider the standard linear model
\begin{align*}
  y = X\theta  + w.
\end{align*}
We assume that $X$ has i.i.d.\ mean zero rows $x_1^\top, \cdots, x_n^\top$
with covariance $\Sigma \in \mathbb{R}^{d \times d}$. The training error of an estimator $\est{\theta} = \est{\theta}(X, y)$, a function
of $X$ and the responses $y$ whose dependence on both we typically
leave implicit,
is $\train_{X, \theta}(\est{\theta}) = \frac 1 n \Ep_{w} [\|X \est{\theta} - y\|_2^2 \mid X, \theta]$, while the prediction (generalization)
error is $\pred_{X, \theta} (\est{\theta}) = \Ep_{x, w} [(x^\top \theta - x^\top \est{\theta})^2 \mid X, \theta] $, where $x$ is an independent copy from the input distribution.
We consider a Bayesian formulation where the ground truth $\theta$ has a prior distribution independent of the data and the noise, and the posterior training and generalization errors are $\train_X(\est{\theta}) = \Ep_\theta [\train_{X, \theta}(\est{\theta})]$ and $\pred_X (\est{\theta}) = \Ep_\theta [\pred_{X, \theta} (\est{\theta})]$.

Given a constraint on the training error $\epsilon \in [0, \infty)$, we can
then formalize the cost of not fitting the training data via the following
optimization problem over a hypothesis class of estimators $\mathcal{H}$.
\begin{equation} \label{eq:not-fitting-problem}
  \begin{aligned}
    \minimize_{\est{\theta} \in \mathcal{H}} &~ \pred_X \prn{\est{\theta}}
    \\
    \subjectto &~ \train_X\prn{\est{\theta}} 
    \geq \epsilon^2 
	\end{aligned}
\end{equation}
Here, the constraint is on the average training error (over $y$);
any estimator that on \emph{each} input $y$ has prescribed
error $\epsilon^2$ immediately satisfies the constraints~\eqref{eq:not-fitting-problem}.
We mainly study the
\textit{cost of not fitting}
\begin{align}
  \cost_X(\epsilon) := \min_{\est{\theta} \in \mathcal{H}(\epsilon)} \pred_X \prn{\est{\theta}} - \min_{\est{\theta} \in \mathcal{H}(0)} \pred_X \prn{\est{\theta}},
  \label{eq:def-cost}
\end{align}
where for a given $\mathcal{H}$ we define the set $\mathcal{H}(\epsilon) \defeq \{\est{\theta} \in
\mathcal{H} \mid \train_X(\est{\theta}) \geq \epsilon^2 \} \subset
\mathcal{H}$.

%
Noting that $\mathcal{H}(t)$ is a decreasing set in $t$, we always have
$\cost_X(\epsilon) \geq 0$. Of course, the best estimator need not necessarily
memorize the entire dataset---as we shall see, some amount of regularization can help---and so we also specifically consider the \textit{cost of not
  interpolating} with respect to the minimum norm interpolating
solution $\thetaols \defeq X^\top (X X^\top)^{-1} y$, defining
\begin{align}
  \wb{\cost}_X(\epsilon) := \min_{\est{\theta} \in \mathcal{H}(\epsilon)} \pred_X \prn{\est{\theta}} - \pred_X \prn{\thetaols}.
  \label{eq:def-cost-interpolator}
\end{align} 

We study problem~\eqref{eq:not-fitting-problem}, in particular through the lens of the quantities~\eqref{eq:def-cost} and~\eqref{eq:def-cost-interpolator}, under the following assumptions.

\begin{assumption}[Proportional asymptotics and spherical prior] \label{assp:proportional}
  The dimension $d := d(n)$ satisfies $d/n \to \gamma \in (1, \infty)$. The
  data matrix $X = [x_1 ~ x_2 ~ \cdots ~ x_n]^\top \in \real^{n \times d}$,
  where $X: = X(n) = (x_{ij}(n))_{i \in [n], j \in [d]}$ forms a triangular
  array of random variables with independent rows. There is a deterministic
  sequence of symmetric positive definite matrices $\Sigma := \Sigma(n) \in
  \real^{d \times d}$ such that $X = Z \Sigma^{\frac 1 2}$, where $Z =
  (z_{ij})_{i \in [n], j \in [d]}$ and $z_{ij}$ are i.i.d. random variables
  with distribution independent of $n$ such that $\Ep [z_{ij}] = 0$,
  $\var(z_{ij}) = 1$, and $\Ep [z_{ij}^4] \leq M$ for a universal constant
  $M$.  In addition, we assume $\theta$ has prior independent
  of the data $X, y$, with zero mean and variance $\var(\theta) = I_d/d$.
\end{assumption}
\noindent
Under Assumption~\ref{assp:proportional},
for each $n$, $x_1(n), \cdots, x_n(n)$ are
i.i.d.\ random vectors such that
\begin{align*}
  \Ep[x_i(n)] = 0, \qquad \var \prn{x_i(n)} = \Sigma(n).
\end{align*} 
Meanwhile, examples of priors satisfying the
assumption include the uniform prior on the unit sphere $\mathbb{S}^{d-1}$
and the Gaussian prior $\normal(0, I_d/d)$, where note that $\Ep [\norm{\theta}_2^2] = 1$.  We assume $\gamma > 1$, and hence, as
the model is overparameterized, zero training error is attainable.

While at first blush appearing restrictive, our
main results characterize the cost of not fitting for linear estimators.

\begin{assumption}[Linear estimators] \label{assp:linear}
  The hypothesis class consists of all linear estimators, i.e.,
  \begin{align*}
    \mathcal{H} = \brc{\est{\theta}(X, y) = Ay\, ,  A := A(X) \in \mathbb{R}^{d \times n}} ,
  \end{align*}
  where $A$ may depend on the features $X$ but not the labels $y$.
\end{assumption}
\noindent
Notably, the hypothesis class of linear estimators contains the popular
ridge estimator $\est{\theta}_{\lambda} \defeq (X^\top X + \lambda I)^{-1}
X^\top y$ and minimum norm interpolant $\thetaols \defeq (X^\top X)^\dagger
X^\top y$. Because we seek exact optimality results for more general
estimators, we follow standard practice in minimax and asymptotic statistics
to choose a prior on the ``true'' parameter $\theta$.  In classical linear
regression, the prior of choice is a Gaussian, so that Anderson's
theorem~(\citeyear{Anderson55}) guarantees the posterior mean is minimax for
any symmetric loss, and so the optimal estimator is linear. In our case, a
similar result holds, though it is more subtle because of the nonconvex
constraint~\eqref{eq:not-fitting-problem} on training error;
Theorem~\ref{thm:strong-duality-hilbert} to come guarantees that when the
prior and noise are both Gaussian, the optimal estimator solving
problem~\eqref{eq:not-fitting-problem} belongs to the collection of linear
estimators.  Thus, our main results extend immediately to the general
class of all square integrable estimators:

\addtocounter{assumption}{-1}
\renewcommand{\theassumption}{A\arabic{assumption}$'$}
\begin{assumption}[Estimators with Gaussian prior]
  \label{assp:general-gaussianity}
  The parameter $\theta \sim \normal(0, I_d/d)$ and the noise
  $w \sim \normal(0,
  \sigma^2 I_n)$.  The hypothesis class consists of measurable, square
  integrable $\est{\theta} : \mathbb{R}^{n \times d + n} \to \mathbb{R}^d$,
  i.e.,
  \begin{align*}
    \mathcal{H} =  \brc{\est{\theta} = \est{\theta}(X,y) \mid
      \Ep_y [\|\est{\theta}(X,y)\|_2^2 \mid X] < \infty} .
  \end{align*}
\end{assumption}
\renewcommand{\theassumption}{A\arabic{assumption}}
\noindent
We return to more discussion in Section~\ref{sec:optimality-Gaussian}.

\section{Main results}

\subsection{The isotropic case}
\label{sec:isotropic}

We first consider the isotropic setting where $\Sigma = I$ for all $n$, and
thus $x_{ij}$ are i.i.d.\ random variables with zero mean and unit
variance. Before stating the main theorem regarding the quantity
$\cost_X(\epsilon)$, we first characterize the optimal solution to the cost
of not fitting problem~\eqref{eq:not-fitting-problem} via strong duality,
illustrating the role random matrix theory plays in computing the optimal
solution value. We postpone most of the technical details to
Section~\ref{proof:cost-isotropic}.

When $\mathcal{H}$ consists of linear estimators $\est{\theta} = Ay$,
we define the shorthand
$\mc{P}(A) \defeq \pred_X(\est{\theta})$ and
$\mc{T}(A) \defeq \train_X(\est{\theta})$, with which we
express the cost of not fitting
problem~\eqref{eq:not-fitting-problem} as
\begin{equation}
  \label{eqn:linear-not-fitting}
  \begin{aligned}
	\minimize_{A \in \real^{d \times n}} & \qquad \mathcal{P}(A) = \frac 1 d \norm{AX-I}_F^2 + \sigma^2 \norm{A}_F^2 \\
	\subjectto & \qquad \mathcal{T}(A) = \frac{1}{nd} \norm{XAX-X}_F^2 + \frac{\sigma^2}{n} \norm{XA-I}_F^2  \geq \epsilon^2 \, .
\end{aligned}
\end{equation}
The problem---while nonconvex---has quadratic
objective and a single quadratic constraint. Thus we may leverage strong
duality~\citep[Appendix~B.1]{BoydVa04}, writing a Lagrangian
and solving, to conclude that for some
$\rho_n := \rho_n(\epsilon)$ such that $I - \frac{\rho_n}{d} X^\top X \succ
0$, the optimal $A$ for the problem~\eqref{eq:not-fitting-problem} is
\begin{align*}
  A(\rho_n) = \prn{I - \rho_n \sigma^2 \prn{I - \frac{\rho_n}{d}
      X^\top X}^{-1}} (X^\top X + d \sigma^2 I)^{-1}X^\top ,
\end{align*}
where $\rho_n$ is the dual optimal value of the Lagrange multiplier
associated with the constraint $\mc{T}(A) \ge \epsilon^2$.  When $\rho_n =
0$, the constraint is inactive, so $A(0)$ is the global minimizer of the
unconstrained problem and evidently corresponds to a ridge regression
estimate; we have $\cost_X(\epsilon) = \mathcal{P}(A(\rho_n)) -
\mathcal{P}(A(0))$ and $\mathcal{T}(A(\rho_n)) = \epsilon^2$. Substituting
$A = A(\rho)$ into $\mathcal{P}(A)$ and $\mathcal{T}(A)$, we obtain
\begin{align*}
	 \mathcal{P}(A(\rho)) - \mathcal{P}(A(0)) & =  \frac{\rho^2  \sigma^4}{d} \Tr \prn{\prn{I - \frac \rho d X^\top X}^{-2} \frac{X^\top X}{d}  \prn{\frac{X^\top X}{d} + \sigma^2 I}^{-1} } \, , \\
	 \mathcal{T}(A(\rho))	& = \frac{\sigma^4}{n} \Tr \prn{\prn{I - \frac \rho d X^\top X}^{-2} \prn{\frac{X^\top X}{d} + \sigma^2 I}^{-1} } \, .
\end{align*}

We may now leverage high-dimensional random matrix theory and asymptotics.
Let $X$ have singular values $\lambda_1 \geq \lambda_2 \geq \cdots \geq
\lambda_n$. Denoting the empirical spectral distribution of $\frac{1}{d}
XX^\top$ via its c.d.f.\ $H_n(s) := \frac 1 n \sum_{i=1}^n
\ind_{\lambda_i^2/d \leq s}$,
we equivalently have
\begin{align*}
	\mathcal{P}(A(\rho)) - \mathcal{P}(A(0)) & = \frac{\rho^2 n}{d} \int \frac{\sigma^4 s}{(1-\rho s)^2 (s + \sigma^2)} dH_n(s) \, , \\
	\mathcal{T}(A(\rho))	& = \int \frac{\sigma^4}{(1 - \rho s)^2 \prn{s + \sigma^2}} dH_n(s) \, .
\end{align*}
By standard results in random matrix theory
(see Lemma~\ref{thm:MP-law}), $H_n$ converges
weakly to the Marchenko-Pastur c.d.f.\ $H$, which has support $[\lambda_-,
  \lambda_+]$ ro $\lambda_{\pm} := \prn{1 \pm 1/\sqrt{\gamma}}^2$, and
density
\begin{align} \label{eq:MP-distribution}
	dH(s) = \frac{\gamma}{2\pi} \frac{\sqrt{(\lambda_+ -s)(s - \lambda_-)}}{s} \ind_{s \in [\lambda_-, \lambda_+]} ds \, .
\end{align}
Therefore for any fixed $0 \le \rho < \frac{1}{1 + \sqrt{\gamma}}$,
\begin{align*}
	\lim_{n \to \infty} \prn{\mathcal{P}(A(\rho)) - \mathcal{P}(A(0))} & = \frac{\rho^2}{\gamma} \int \frac{\sigma^4 s}{(1-\rho s)^2 (s + \sigma^2)} dH(s)  \, , \\
	\lim_{n \to \infty} \mathcal{T}(A(\rho))	& = \int \frac{\sigma^4}{(1 - \rho s)^2 \prn{s + \sigma^2}} dH(s) \, .
\end{align*}
Setting $\rho = 0$ corresponds to making the
constraint~\eqref{eqn:linear-not-fitting} inactive, so we therefore define
the memorization threshold
\begin{equation}
  \label{eq:threshold-isotropic}
  \epsilon_\sigma^2 \defeq \int \frac{\sigma^4}{s +\sigma^2} dH(s),
\end{equation}
and observe that for any $\epsilon^2 \geq \epsilon_\sigma^2$, there exists a
$\rho \geq 0$ such that $\lim_{n \to \infty} \mathcal{T}(A(\rho)) =
\epsilon^2$. Given that $\mathcal{T}(A(\rho_n)) = \epsilon^2$, we
expect that $\lim_{n \to \infty} \rho_n = \rho$ and
therefore should have
\begin{equation*}
  \lim_{n \to \infty} \cost_X(\epsilon) = \lim_{n \to \infty}
\prn{\mathcal{P}(A(\rho)) - \mathcal{P}(A(0))} = \frac{\rho^2}{\gamma} \int
\frac{\sigma^4 s}{(1-\rho s)^2 (s + \sigma^2)} dH(s).
\end{equation*}
We can make each of these steps rigorous (see
Section~\ref{proof:cost-isotropic}), yielding the following theorem.
\begin{theorem} \label{thm:cost-isotropic}
  Let Assumption~\ref{assp:proportional} and either Assumption~\ref{assp:linear} or \ref{assp:general-gaussianity} hold. Then as $n
  \to \infty$,
  \begin{enumerate}[label=(\roman*),leftmargin=1.5em]
  \item \label{item:threshold-value}
    (\textbf{threshold value}) for $\epsilon_\sigma$ defined in
    Eq.~\eqref{eq:threshold-isotropic}, $\epsilon_\sigma^2 =
    \frac{\sigma^4}{\sigma^2 + 1 -1/\gamma} + o(\sigma^4)$.
  \item \label{item:no-cost-below}
    (\textbf{no cost below threshold}) if $\epsilon < \epsilon_\sigma$,
    then with probability one $\lim_{n \to \infty} \cost_X(\epsilon) = 0$.
    In addition, for the ridge estimator $\est{\theta}_{d\sigma^2 } =
    (X^\top X + d\sigma^2 I)^{-1} X^\top y$, we have
    \begin{align*}
      \lim_{n \to \infty} \prn{\min_{\est{\theta} \in \mathcal{H}(\epsilon)}  \pred_X \prn{\est{\theta}} - \pred_X \prn{\est{\theta}_{d\sigma^2 }}} = 0 \, .
    \end{align*}
  \item \label{item:cost-not-fitting}
    (\textbf{cost of not fitting}) if $\epsilon \geq \epsilon_\sigma$,
    there exists a scalar $\rho:= \rho(\epsilon) \in
    \openright{0}{\lambda_+^{-1}}$ that uniquely solves
    \begin{align}
      \label{eq:rho-epsilon-equation}
      \int \frac{\sigma^4}{(1 - \rho s)^2 \prn{s + \sigma^2}} dH(s) =
      \epsilon^2,
    \end{align}
    and with probability one
    \begin{align}
      \lim_{n \to \infty} \cost_X(\epsilon) = \frac{\rho^2}{\gamma} \int \frac{\sigma^4 s}{(1-\rho s)^2 (s + \sigma^2)} dH(s).
      \label{eq:cost-limit-above-threshold}
    \end{align}
    For the constants $\mathsf{c} := \frac{2}{\lambda_-^2 +
      \sigma^2}$ and $\mathsf{C} \defeq \frac{(1 -
      1/\sqrt{2})^2 \lambda_-}{\lambda_+^2 \gamma}$, we have $\lim_{n \to
      \infty} \cost_X(\epsilon) \geq \mathsf{C} \epsilon^2$ whenever
    $\epsilon^2 \geq \mathsf{c} \sigma^4$.
  \end{enumerate}
\end{theorem}

Part~\ref{item:threshold-value} of Theorem~\ref{thm:cost-isotropic}
characterizes the threshold for the constraint on training error above which
no linear estimator can achieve optimal generalization; from
part~\ref{item:no-cost-below}, so long as the constraint is below this
threshold, optimal generalization remains attainable.  Together, parts
\ref{item:threshold-value} and \ref{item:cost-not-fitting} of the theorem
imply that for an estimator to achieve optimal generalization, the estimator
must incur $O(\sigma^4)$ training error as the label noise variance
$\sigma^2$ tends to zero. When $\sigma^2$ is small, this is quadratically
smaller than the inherent noise floor in the problem. Moreover,
part~\ref{item:cost-not-fitting} implies eventually for sufficiently large
$\epsilon$ that $\cost_X(\epsilon)$ grows linearly in terms of the
constraint on training error $\train_X(\est{\theta}) = \epsilon^2$---by not
memorizing, we are essentially paying the same additional amount of error in
generalization in terms of training error up to a constant factor.  We
conclude that memorization for high dimensional linear regression---training
to accuracy quadratically smaller than the inherent noise floor in the
problem---is necessary, and with the ``necessity'' increasing as the
signal-to-noise ratio grows.

We now turn to look specifically at the cost of exact interpolation; instead
of comparing against the best linear estimator, we
characterize $\wb{\cost}_X(\epsilon)$ (see
Eq.~\eqref{eq:def-cost-interpolator}), the prediction error of $\est{\theta}
\in \mc{H}(\epsilon)$ to the minimum norm interpolant
$\thetaols$.  We provide a proof of the following theorem in
Appendix~\ref{proof:cost-interpolator-isotropic}.

\begin{theorem} \label{thm:cost-interpolator-isotropic}
  Let Assumption~\ref{assp:proportional} and either Assumption~\ref{assp:linear} or \ref{assp:general-gaussianity} hold. Then 
  \begin{enumerate}[label=(\roman*),leftmargin=1.5em]
  \item \label{item:relative-costs-ols}
    (\textbf{interpolation cost})
    for any $\epsilon \geq 0$,
    $\cost_X(\epsilon) - \wb{\cost}_X(\epsilon)
    = \pred_X (\thetaols) -
    \pred_X(\est{\theta}(0))$, and
    with probability one
    \begin{align*}
      \lim_{n \to \infty} \prn{
        \pred_X(\thetaols)
        - \pred_X(\est{\theta}(0))}
      = \frac{\sigma^4}{\gamma} \int \frac{1}{s(s+\sigma^2)} dH(s)
      =  \frac{\sigma^4}{\gamma \prn{1 - 1/\gamma}^3} + o(\sigma^4).
    \end{align*}
  \item \label{item:interp-threshold-ols}
    (\textbf{interpolation threshold})
    for any $\sigma > 0$, there exists a $\rho = \rhools \in (0,
    \lambda_+^{-1})$ that uniquely solves
    \begin{align}
      \label{eq:rho-sigma-equation}
      \rho^2 \int \frac{s}{(1 - \rho s)^2 \prn{s + \sigma^2}} dH(s)
      = \int \frac{1}{s(s+\sigma^2)} dH(s),
    \end{align} 
    where for the threshold $\epsols^2 \defeq \int \frac{\sigma^4}{(1 - \rhools
      s)^2 \prn{s + \sigma^2}} dH(s)$ we have
    \begin{equation*}
      \lim_{n \to \infty}
      \wb{\cost}_X(\epsilon) ~ \begin{cases}
        < 0 & \mbox{if~} \epsilon < \epsols \\
        = 0 & \mbox{if~} \epsilon = \epsols \\
        > 0 & \mbox{if}~ \epsilon > \epsols.
      \end{cases}
    \end{equation*}
    In comparison to the
    threshold $\epsilon_\sigma$ in Eq.~\eqref{eq:threshold-isotropic}
    and Theorem~\ref{thm:cost-isotropic}, we have
    $\epsilon_\sigma < \epsols \le \frac{2
      \lambda_+}{\lambda_-}\epsilon_\sigma$.
  \end{enumerate}
\end{theorem}

Part~\ref{item:relative-costs-ols} shows that the minimum norm interpolant
is nearly optimal, at least as $\sigma^2 \to 0$: its prediction error over
the best (linear) estimator scales asymptotically as $O(\sigma^4 / \gamma)$,
and as the aspect ratio $\gamma$ increases it becomes closer and closer to
optimal.  Part~\ref{item:interp-threshold-ols} complements this result,
showing that if the constraint $\epsilon$ on the training error of an
estimator is at most $\epsilon^2 \le \epsols^2 = O(\sigma^4)$, there are
better estimators than the minimum norm interpolant; one concrete example
here is the optimal ridge estimator $\est{\theta}_{d \sigma^2}$, which has
asymptotic training error, as we see from Theorem~\ref{thm:cost-isotropic}.

\subsection{Features with general covariance}

In this section, we develop analogous results to those for the identity
covariance in Sec.~\ref{sec:isotropic}, showing that the results are not
merely some fragile and magical consequences of isotropy.  Here, we make the
following assumption about the covariance matrix $\Sigma$.
\begin{assumption} \label{assp:cov-spectral}
  The population covariance $\Sigma$ has eigenvalues $t_1 \geq t_2
  \geq \cdots \geq t_d \geq 0$, where $t_1 = 1$ and there exists
  $\kappa < \infty$ such that $t_d \geq 1/\kappa$. The
  empirical spectral distribution
  $T_n(s) \defeq \frac{1}{d} \sum_{i = 1}^d \ind_{t_i \le s}$ of
  $\Sigma$
  converges weakly to a c.d.f.\ $T$.
\end{assumption}
Under this assumption, the empirical distribution for the eigenvalues of
$\frac 1 d XX^\top$ converges weakly to a distribution with deformed
Marchenko-Pastur c.d.f.\ $G$. (See Lemma~\ref{thm:MP-deformed-law} for the
precise definition.)  With the limit $G$ and recalling the Marchenko-Pastur
c.d.f.\ $H$, we may characterize $\cost_X(\epsilon)$ for general covariances
$\Sigma$. The result is analogous to Theorem~\ref{thm:cost-isotropic},
modulo the condition number $\kappa$ and the alternative limit $G$. To that
end, define the deformed threshold
\begin{equation}
  \epsdeformed^2 \defeq \int \frac{\sigma^4}{s + \sigma^2} dG(s),
  \label{eqn:threshold-anisotropic}
\end{equation}
comparing to the definition~\eqref{eq:threshold-isotropic} of
$\epsilon_\sigma = \int\frac{\sigma^4}{s + \sigma^2} dH(s)$.  We then have
the following theorem, whose proof we provide in
Appendix~\ref{proof:cost-general-cov}.

\begin{theorem} \label{thm:cost-general-cov}
  Let Assumptions~\ref{assp:proportional} and
  \ref{assp:cov-spectral} hold, $\sigma > 0$, and let $G$ be the deformed
  Marchenko-Pastur c.d.f.\ in Lemma~\ref{thm:MP-deformed-law}. If either Assumption~\ref{assp:linear} or \ref{assp:general-gaussianity} holds, then as $n \to \infty$,
  \begin{enumerate}[label=(\roman*),leftmargin=1.5em]
  \item \label{item:general-threshold}
    (\textbf{threshold value}) for $\epsdeformed$ defined
    in Eq.~\eqref{eqn:threshold-anisotropic},
    $\epsdeformed^2
    \leq \epsilon_{\sqrt{\kappa}\sigma}^2 / \kappa = \frac{\kappa
      \sigma^4}{\kappa \sigma^2 + 1 -1/\gamma} + o(\sigma^4)$.
  \item (\textbf{no cost below threshold}) if $\epsilon <
    \epsdeformed$, then with probability one $\lim_{n \to \infty}
    \cost_X(\epsilon) = 0$.  In addition, define the ridge estimator
    $\est{\theta}_{d\sigma^2 } = (X^\top X + d\sigma^2 I)^{-1} X^\top y$, we
    have
    \begin{equation*}
      \lim_{n \to \infty} \prn{\min_{\est{\theta} \in \mathcal{H}(\epsilon)}  \pred_X \prn{\est{\theta}} - \pred_X \prn{\est{\theta}_{d\sigma^2 }}} = 0.
    \end{equation*}
  \item \label{item:cost-not-fit-general}
    (\textbf{cost of not fitting}) If $\epsilon \geq
    \epsdeformed$, there exists $\rhodeformed =
    \rhodeformed(\epsilon) \in \openright{0}{1/\lambda_+}$ that uniquely solves
      \begin{align}
        \label{eq:rho-epsilon-equation-general-cov}
	\kappa \sigma^4 \cdot \prn{\int \frac{ 1}{(1 - \rho s)^2
            \prn{s + \kappa \sigma^2}} dH(s)
          - \int \frac{ 1}{ s + \kappa \sigma^2} dH(s)}
        = \epsilon^2 - \epsdeformed^2,
      \end{align}
      where $H$ is the Marchenko-Pastur c.d.f.~\eqref{eq:MP-distribution}.
      Further, with probability one
      \begin{align*}
	\liminf_{n \to \infty} \cost_X(\epsilon) \geq
        \frac{\rhodeformed^2}{\gamma} \int \frac{\sigma^4 s}{
          (1 - \rhodeformed s)^2 (s + \sigma^2)} dH(s).
      \end{align*}
      For the constants
      $\mathsf{c} \defeq \frac{2\kappa}{\lambda_- + \kappa \sigma^2}$
      and $\mathsf{C} = \frac{\lambda_- (1 - 1/\sqrt{2})^2}{\kappa \lambda_+^2
        \gamma}$,
      we have
      $\liminf_{n \to \infty} \cost_X(\epsilon) \ge \mathsf{C}
      \epsilon^2$
      whenever $\epsilon^2 \geq \mathsf{c} \sigma^4$.
  \end{enumerate}
\end{theorem}
 



\subsection{Optimality of general estimators in Gaussian case} \label{sec:optimality-Gaussian}

While, as we discuss before Assumption~\ref{assp:general-gaussianity},
the lower bounds in Theorems~\ref{thm:cost-isotropic},
\ref{thm:cost-interpolator-isotropic}, and~\ref{thm:cost-general-cov} apply
over the class of linear estimators, which allows our exact
predictive risk characterizations, these results hold for all
estimators satisfying mild regularity conditions under a Gaussianity
assumption on the data distribution.  
Our main insight here is that when the
prior and noise distributions are Gaussian, for all $\epsilon \ge 0$, the
linear estimator class contains the optimal
estimator among the broader class of all square integrable estimators
with training error at least $\epsilon^2$. Of course, this is trivial
when $\epsilon = 0$, as given $(X, y)$ in such a model,  the posterior
on $\theta$ is Gaussian. That the result holds for $\epsilon > 0$ is a
bit more subtle. Specifically, we
have the following theorem, whose proof we provide
in Appendix~\ref{proof:strong-duality-hilbert}.



\begin{theorem} \label{thm:strong-duality-hilbert}
  Let Assumptions~\ref{assp:proportional} and
  \ref{assp:general-gaussianity} hold.
  Let $\Hlin$ and $\Hsq$
  denote the classes of linear and square integrable estimators in Assumptions~\ref{assp:linear} and \ref{assp:general-gaussianity}, respectively.
    Then for all $\epsilon \geq 0$,
  $\inf_{\est{\theta} \in \Hsq(\epsilon)} \pred_X (\est{\theta}) =
  \min_{\est{\theta} \in \Hlin(\epsilon)} \pred_X (\est{\theta})$.
\end{theorem}

Observing in the Gaussian case that the posterior over $\theta \mid y$ is
has mean linear in $y$ and covariance independent of $y$, the
main idea underlying the
proof is to factor the prediction and training error over the marginal
distribution of $y$, as
\begin{align}
	\pred_X \prn{\est{\theta}} & = \Ep_y \brk{\Ep_{\theta \mid y} \brk{\norm{\Sigma^{\frac 1 2} \prn{\est{\theta}(X,y) - \theta}}_2^2 \,\Big\vert\, y} \,\Big\vert\, X} \nonumber \\
	\train_X\prn{\est{\theta}} & = \Ep_y \brk{\normtwo{X \est{\theta}(X,y) - y}^2 \,\Big\vert\, X} \nonumber \, .
\end{align}

Thus the cost of not fitting problem~\eqref{eq:not-fitting-problem} is a
functional (infinite-dimensional)
optimization problem over $\mathcal{H}_\textup{sq}$, with a quadratic
objective and a single quadratic constraint, for which we show that
strong duality still obtains. Applying the appropriate Karush-Kuhn-Tucker
conditions, we can then recover that the optimal estimator is linear, and
in particular is
\begin{equation*}
    \est{\theta}(X, y) = \prn{I - \rho(\epsilon) \sigma^2 \prn{\Sigma - \frac{\rho(\epsilon)}{d} X^\top X}^{-1}}  (X^\top X + d \sigma^2 I)^{-1} X^\top y.
\end{equation*}
Here $\rho(\epsilon)$ is the dual optimal value of the Lagrange multiplier
for the constraint on training error, and it is identical to that in
Theorems~\ref{thm:cost-isotropic}
and~\ref{thm:cost-general-cov}.  See
Section~\ref{section:strong-duality} for the details.

%% file: sections/proof-isotropic.tex
\section{Proof of Theorem~\ref{thm:cost-isotropic}} \label{proof:cost-isotropic}

\subsection{Reduction by strong duality}
\label{section:strong-duality}

We first provide some technical lemmas to reduce the nonconvex
problem~\eqref{eq:not-fitting-problem}. The lemmas will be useful in both
the isotropic case and the general covariance case, and in particular the
key ingredient that allows for this reduction is strong duality in quadratic
optimization.

The first lemma gives an equivalent formulation of the cost of not fitting
problem~\eqref{eq:not-fitting-problem} using the closed forms of $\pred_X
(\est{\theta})$ and $\train_X (\est{\theta})$. We defer the proof to
Appendix~\ref{proof:closed-form-errors}.
\begin{lemma} \label{lem:closed-form-errors}
  Let Assumption~\ref{assp:linear} hold and assume $X = Z \Sigma^{\frac 1
    2}$. Then for any $\est{\theta} \in \mathcal{H}$ the following
  is an equivalent formulation of problem~\eqref{eq:not-fitting-problem}:
  \begin{equation} \label{eq:core-non-convex}
    \begin{aligned}
      \minimize_{A \in \real^{d \times n}} & \qquad \mathcal{P}(A; \Sigma) := \frac{1}{d} \norm{\Sigma^{\frac{1}{2}} \prn{AX-I}}_F^2 + \sigma^2 \norm{\Sigma^{\frac{1}{2}}A}_F^2 \\
      \subjectto & \qquad \mathcal{T}(A; \Sigma) := \frac{1}{nd} \norm{XAX-X}_F^2 + \frac{\sigma^2}{n} \norm{XA-I}_F^2  \geq \epsilon^2.
	\end{aligned}
\end{equation}
\end{lemma}
As strong duality holds for this problem~\citep[cf.][Appendix~B.1]{BoydVa04}, we derive in Lemma~\ref{lem:strong-duality} the optimality criteria via studying the dual. We postpone the proof details to Appendix~\ref{proof:strong-duality}.
\begin{lemma} \label{lem:strong-duality}
  There exists a $\rho_n := \rho_n(\epsilon, \Sigma) \geq 0$ such that
  $\Sigma - \frac{\rho_n}{d} X^\top X \succ 0$ and the optimal solution of
  problem~\eqref{eq:core-non-convex} is $A_n:=A(\rho_n,
  \Sigma)$,
  where
  \begin{subequations}
    \label{eqn:A}
    \begin{align}
      A(\rho, \Sigma)
      & = \prn{I - \rho \sigma^2 \prn{\Sigma - \frac{\rho}{d} X^\top X}^{-1}} X^\top (XX^\top + d \sigma^2 I)^{-1} \label{eq:A-1} \\
      & =
      \prn{I - \rho \sigma^2 \prn{\Sigma - \frac{\rho}{d} X^\top X}^{-1}}  (X^\top X + d \sigma^2 I)^{-1} X^\top. \label{eq:A-2}
    \end{align}
  \end{subequations}
  $A(\rho, \Sigma)$ is defined for $\rho \in D$, where $D$ is the interval
  for all $\rho \geq 0$ such that $\Sigma - \frac{\rho}{d} X^\top X \succ
  0$.
\end{lemma}
\noindent
We suppress the dependence of $A, \rho$ on the data matrix $X$ for  simplicity.

In the next lemma we derive the exact forms of the constraint $\mathcal{T}(A(\rho, \Sigma); \Sigma)$ and the growth of the objective $\mathcal{P}(A(\rho, \Sigma); \Sigma) - \mathcal{P}(A(0, \Sigma); \Sigma)$. We defer the proof to Appendix~\ref{proof:rho-error-forms}.

\begin{lemma} \label{lem:rho-error-forms}
  Let the conditions of Lemma~\ref{lem:strong-duality} hold, and assume
  $XX^\top$ is non-singular. Then for any $\rho \in D$ we have
  \begin{align*}
    \lefteqn{\mathcal{P}(A(\rho, \Sigma); \Sigma) - \mathcal{P}(A(0, \Sigma); \Sigma)} \\
    & =  \frac{\rho^2  \sigma^4}{d} \Tr \prn{\prn{\Sigma - \frac \rho d X^\top X}^{-1}\Sigma \prn{\Sigma - \frac \rho d X^\top X}^{-1} X^\top X  \prn{X^\top X + d \sigma^2 I}^{-1} },
  \end{align*}
  and
  \begin{align*}
    \lefteqn{\mathcal{T}(A(\rho, \Sigma); \Sigma)} \\
    & = \frac{d \sigma^4}{n} \Tr \prn{\Sigma \prn{\Sigma - \frac \rho d X^\top X}^{-1} X^\top X \prn{\Sigma - \frac \rho d X^\top X}^{-1} \Sigma \prn{X^\top X}^\dagger \prn{X^\top X + d\sigma^2 I}^{-1} }. \nonumber
  \end{align*}
\end{lemma}


\subsection{Main proof of Theorem~\ref{thm:cost-isotropic}}

By Theorem~\ref{thm:strong-duality-hilbert}, we only need to prove under Assumption~\ref{assp:linear} with the linear hypothesis class $\mathcal{H} = \{\est{\theta}: \est{\theta} = A(X) y\}$.

\paragraph{Part I: Memorization threshold.} From Eq.~\eqref{eq:threshold-isotropic}, we can directly write out
\begin{align}
	\epsilon_\sigma^2 & = \int \frac{\sigma^4}{s +\sigma^2} dH(s) = \sigma^4 \cdot \lim_{y \to 0^+} m_H(-\sigma^2 + i y) \, , \label{eq:using-stieljes-transform}
\end{align}
where $m_H : \mathbb{C}_+ \to \mathbb{C}_+$ is the Stieltjes transform
(cf.~\eqref{eq:stieltjes-transform}) of the Marchenko-Pastur law.
\begin{lemma} \label{lem:cost-isotropic-mid-1}
	For any $\sigma^2 > 0$,
	\begin{align*}
		\lim_{y \to 0^+} m_H(-\sigma^2 + i y) = \frac{\sigma^2  + o(\sigma^2)}{\sigma^2 \cdot \prn{1 - 1/\gamma + \sigma^2}}.
	\end{align*}
\end{lemma}
We defer the proof to Appendix~\ref{proof:cost-isotropic-mid-1}. We conclude the proof of (i) by applying Lemma~\ref{lem:cost-isotropic-mid-1} to Eq.~\eqref{eq:using-stieljes-transform},
\begin{align*}
	\epsilon_\sigma^2 & = \sigma^4 \cdot \frac{\sigma^2 + o(\sigma^2)}{\sigma^2 \cdot \prn{1 - 1/\gamma + \sigma^2}} = \frac{\sigma^4}{\sigma^2 + 1 - 1/\gamma} + o(\sigma^4) \, .
\end{align*}

\paragraph{Part II: No cost below threshold.} Invoke Lemma~\ref{lem:strong-duality} and set $\rho = 0$ (when the constraint is not active)
to obtain the global minimizer for the unconstrained problem
\begin{align*}
  A(0, I)  = X^\top (XX^\top + d \sigma^2 I)^{-1}
  = (X^\top X + d \sigma^2 I)^{-1} X^\top,
\end{align*} 
so the ridge estimator $\est{\theta}_{d\sigma^2}$ is optimal in
$\mathcal{H}(0)$. Thus we must prove that $\est{\theta}_{d\sigma^2} \in
\mathcal{H}(\epsilon)$ eventually, for which it suffices to show
\begin{align*}
	\liminf_{n \to \infty}	\train_X\prn{\est{\theta}_{d \sigma^2}} = \liminf_{n \to \infty} \mathcal{T}(A(0, I); I)> \epsilon^2 \, ,
\end{align*}
where $\mathcal{T}(A; \Sigma)$ is defined in
Eq.~\eqref{eq:core-non-convex}. When $\Sigma = I$, we can compute the exact
limits in Lemma~\ref{lem:rho-error-forms} when $n \to \infty$.
\begin{lemma} \label{lem:cost-isotropic-mid-2}
  Fix $0 \le \rho < \lambda_+^{-1}$. Then with probability one
  \begin{align*}
    \lim_{n \to \infty} \prn{\mathcal{P}(A(\rho, I); I) - \mathcal{P}(A(0, I); I)} & =  \frac{\rho^2  }{\gamma} \int\frac{\sigma^4 s}{(1-\rho s)^2 (s + \sigma^2)} dH(s) , \\
    \lim_{n \to \infty}  \mathcal{T}(A(\rho,I); I)& = 	\int \frac{\sigma^4 }{(1 - \rho s)^2(s + \sigma^2)} dH(s) .
  \end{align*}
\end{lemma}
Invoke Lemma~\ref{lem:cost-isotropic-mid-2} above for $\rho = 0$
to conclude that with probability one
\begin{align*}
  \lim_{n \to \infty} \mathcal{T}(A(0, I); I) &= \lim_{n \to \infty}\int \frac{\sigma^4 }{s + \sigma^2} dH_n(s)  = \int \frac{\sigma^4 }{s + \sigma^2} dH(s) = \epsilon_\sigma^2 > \epsilon^2.
\end{align*}

\paragraph{Part III: Cost of not-fitting above threshold.}
First we show for any $\epsilon \geq \epsilon_\sigma$ there exists a unique
$\rho = \rho(\epsilon) \in \openright{0}{\lambda_+^{-1}}$ that solves the
fixed point~\eqref{eq:rho-epsilon-equation}, i.e.
\begin{align*}
  \int \frac{\sigma^4}{(1 - \rho s)^2 \prn{s + \sigma^2}} dH(s) = \epsilon^2.
\end{align*}
As the left hand side is increasing in $\rho$ and when $\rho
\downarrow 0$, the integral approaches $\epsilon_\sigma^2
= \int \frac{\sigma^4}{s + \sigma^2} dH(s)$. On the other
hand, by substituting in the exact formula of $dH(s)$ in
Eq.~\eqref{eq:MP-distribution}, we see as $s \uparrow \lambda_+$,
\begin{align}
  \frac{\sigma^4}{(1 - \lambda_+^{-1} s)^2 \prn{s + \sigma^2}} dH(s) 
  & = (1 + o(1)) \frac{\gamma \lambda_+ \sigma^4 \sqrt{\lambda_+ - \lambda_-}}{
    2\pi \prn{\lambda_+ + \sigma^2}} \cdot \prn{\lambda_+ - s}^{-\frac{3}{2}} ds,
  \label{eq:integral-blows-up}
\end{align}
so that the improper integral diverges when $\rho =
\lambda_+^{-1}$. Monotone convergence then implies that the integral
approaches $\infty$ as $\rho \uparrow \lambda_+^{-1}$.

It remains to show the limiting
statement~\eqref{eq:cost-limit-above-threshold} in
part~\ref{item:cost-not-fitting} of the theorem and the growth lower bounds.
To do so, we leverage the duality calculations in
Lemma~\ref{lem:strong-duality}
to transfer between the training error $\epsilon$ and the Lagrange multiplier
$\rho$, using that to construct upper and lower bounds on $\cost_X(\epsilon)$.
By Lemma~\ref{lem:strong-duality}, the estimator
\begin{equation*}
  \est{\theta}(\rho) \defeq A(\rho, I) y
\end{equation*}
is the optimal
solution to problem~\eqref{eq:core-non-convex} when $\epsilon^2 =
\mathcal{T}(A(\wb{\rho}, I); I)$, that is, $A(\rho, I)$ solves
\begin{equation*}
  \minimize_{A \in \R^{d \times n}} ~
  \mc{P}(A; I) ~~~ \subjectto ~~
  \mc{T}(A; I) \ge \mc{T}(A(\rho, I); I).
\end{equation*}
Thus, whenever $\mc{T}(A(\rho, I); I) < \epsilon^2$ it holds that
\begin{subequations}
  \label{eqn:cost-upper-lower-rho}
  \begin{equation}
    \cost_X(\epsilon) \ge \mc{P}(A(\rho, I); I) - \mc{P}(A(0, I); I)
  \end{equation}
  while when $\mc{T}(A(\rho, I); I) > \epsilon^2$, it holds that
  \begin{equation}
    \cost_X(\epsilon) \le \mc{P}(A(\rho, I); I) - \mc{P}(A(0, I); I).
  \end{equation}
\end{subequations}
We will give matching upper and lower bounds to the
quantities~\eqref{eqn:cost-upper-lower-rho} to show the
limit~\eqref{eq:cost-limit-above-threshold}.

Let $\rho(\epsilon) \in (0, \lambda_+^{-1})$ be the $\rho$ satisfying the
fixed point~\eqref{eq:rho-epsilon-equation}, where $\rho(\epsilon) > 0$ as
$\epsilon^2 > \epsilon_\sigma$ by assumption (as otherwise $\lim_n
\cost_X(\epsilon) = 0$ by part~\ref{item:no-cost-below} of the theorem).
For any $\rho \in \openright{0}{\lambda_+^{-1}}$,
Lemma~\ref{lem:cost-isotropic-mid-2} implies
\begin{align*}
  \lim_{n \to \infty} \mathcal{T}(A(\rho, I); I)
  = \int \frac{\sigma^4 }{(1 - \rho s)^2(s + \sigma^2)} dH(s).
\end{align*}
Then $\rho > \rho(\epsilon)$ implies that
$\lim_n \mc{T}(A(\rho, I); I) > \epsilon^2$, while
$\rho < \rho(\epsilon)$ implies that
$\lim_n \mc{T}(A(\rho, I); I) < \epsilon^2$. In particular, the
inequalities~\eqref{eqn:cost-upper-lower-rho} and these limits
on $\mc{T}$ combine to give that
\begin{align*}
  \limsup_{n \to \infty} \cost_X(\epsilon)
  & \leq
  \liminf_{n \to \infty} \left[
    \mc{P}(A(\rho, I); I) - \mc{P}(A(0, I); I)\right]
\end{align*}
whenever $\rho > \rho(\epsilon)$, while if $\rho < \rho(\epsilon)$ we have
\begin{align*}
  \liminf_{n \to \infty} \cost_X(\epsilon)
  & \ge
  \limsup_{n \to \infty} \left[
    \mc{P}(A(\rho, I); I) - \mc{P}(A(0, I); I)\right].
\end{align*}
We can now apply the limiting expansion
of $\mc{P}(A(\rho)) - \mc{P}(A(0))$ in
Lemma~\ref{lem:cost-isotropic-mid-2}, which yields
that for any $0 \le \rho_0 < \rho(\epsilon) < \rho_1 < \lambda_+^{-1}$, we have
\begin{align*}
  \lefteqn{\frac{\rho_0^2  }{\gamma}
    \int\frac{\sigma^4 s}{(1 - \rho_0 s)^2 (s + \sigma^2)} dH(s)
    = \lim_{n \to \infty} \left[\mc{P}(A(\rho_0, I); I)
      - \mc{P}(A(0, I); I)\right]} \\
  & \qquad \le 
  \liminf_{n \to \infty} \cost_X(\epsilon)
  \le \limsup_{n \to \infty} \cost_X(\epsilon) \\
  & \qquad
  \le \lim_{n \to \infty} \left[\mc{P}(A(\rho_1, I); I)
    - \mc{P}(A(0, I); I)\right]
  =
  \frac{\rho_1^2  }{\gamma}
  \int\frac{\sigma^4 s}{(1 - \rho_1 s)^2 (s + \sigma^2)} dH(s)
\end{align*}
Take $\rho_1 \downarrow \rho(\epsilon)$ and $\rho_0 \uparrow \rho(\epsilon)$
to obtain the limit~\eqref{eq:cost-limit-above-threshold}.

We complete the proof of part~\ref{item:cost-not-fitting} of the theorem via
the following final lemma, which provides a linear lower bound for $\lim_{n
  \to \infty}\cost_X(\epsilon)$.

\begin{lemma} \label{lem:cost-isotropic-mid-3}
  Let $\mathsf{c} = \frac{2}{\lambda_-^2 + \sigma^2}$. If
  $\epsilon^2 \ge \mathsf{c} \sigma^4$, then
  \begin{align*}
    \lim_{n \to \infty}
    \cost_X(\epsilon) \ge \frac{(1 - 1/\sqrt{2})^2}{\lambda_+^2}
    \frac{\lambda_-}{\gamma} \cdot \epsilon^2.
  \end{align*}
\end{lemma}
\begin{proof}
  Taking $\rho$ to solve the fixed point~\eqref{eq:rho-epsilon-equation},
  the limit~\eqref{eq:cost-limit-above-threshold} yields
  \begin{equation*}
    \lim_n \cost_X(\epsilon)
    \stackrel{\eqref{eq:cost-limit-above-threshold}}{=}
    \frac{\rho^2 \sigma^4}{\gamma} \int \frac{s}{(1 - \rho s)^2 (s + \sigma^2)}
    dH(s)
    \ge \frac{\rho^2 \lambda_-}{\gamma}
    \int \frac{\sigma^4}{(1 - \rho s)^2 (s + \sigma^2)}
    dH(s) \stackrel{\eqref{eq:rho-epsilon-equation}}{=}
    \frac{\rho^2 \lambda_-}{\gamma} \epsilon^2,
  \end{equation*}
  Thus it suffices to show that $\rho \ge \frac{1}{\lambda_+}
  (1 - 1/\sqrt{2})$.
  To see this, we leverage the following
  inequalities:
  \begin{align*}
    \frac{1}{(1 - \rho \lambda_+)^2}
    \ge \int \frac{1}{(1 - \rho s)^2} dH(s)
    \ge \int \frac{\lambda_- + \sigma^2}{(1 - \rho s)^2 (s + \sigma^2)} dH(s)
    = \frac{\lambda_- + \sigma^2}{\sigma^4} \epsilon^2
    \ge 2,
  \end{align*}
  the last inequality holding for
  $\epsilon^2 \ge \frac{2 \sigma^4}{\sigma^2 + \lambda_-}$.
  Rearranging $(1 - \rho \lambda_+)^2 \le \half$ yields
  $\rho \ge \frac{1}{\lambda_+}(1 - 1/\sqrt{2})$, which implies the
  claimed result.
\end{proof}

%% file: sections/discussion.tex
\section{Discussion}

By characterizing the  excess prediction error in linear regression
models as a function of constraints on training error, this paper
gives insights into the necessity---in achieving
optimal prediction risk---of memorization for learning.
Our results support the natural conclusion that interpolation is
particularly beneficial in settings with low label noise, which as we note
earlier, may include some of the most widely-used existing benchmarks for
deep learning. Even more, they suggest that---at least when the noise is
low---memorization may simply be necessary, so that a deeper understanding
of the generalization of modern machine learning algorithms may require a
careful look at more precise noise properties of the prediction problems at
hand.

In the anisotropic setting, our lower bounds on prediction error depend on the condition number of the data covariance, and thus our bounds not apply, i.e., are vacuous, in settings such as sparse covariance or kernel regression.
Extending our results to these settings is an interesting direction for future work.
Furthermore, our analysis relies heavily on the fact that both the prediction and empirical risk are quadratic in the case of least-squares regression, and thus strong duality obtains.
Proving similar results in settings such as linear binary classification, where the optimal unconstrained estimator, i.e., margin maximizing solution, is nonlinear and the risk no longer quadratic, is an exciting open problem.

%% file: sections/matrix-review.tex
\section{Asymptotics of random matrices}
\label{section:matrix-review}

In this appendix, we review the classical results regarding singular values
of random matrices we require. Consider a triangular array of independent
and identically distributed random variables $(z_{ij}(n))_{i\in[n],
  j\in[d]}$ for $n=1,2,\cdots$ and $d:=d(n)$. We write $Z := Z(n)=
(z_{ij}(n)) \in \mathbb{R}^{n \times d}$. Throughout we assume the
proportional asymptotics $d/n \to \gamma \in (1, \infty)$, so the matrices
$Z$ have rank at most $n$.
We assume throughout that the entries $z_{ij}$  satisfy
$\E[z_{ij}] = 0$ and $\E[z_{ij}^2] = 1$.
We have the following standard
Marchenko-Pastur and Bai-Yin laws.

\begin{lemma}[Marchenko-Pastur law, \citet{BaiSi10}, Thm.~3.4]
  \label{thm:MP-law}
  Let $Z$ have singular values $\lambda_1 \ge \lambda_2 \ge \cdots \ge
  \lambda_n \ge 0$, and let $\frac{1}{d} ZZ^\top$ have spectral
  distribution with c.d.f.
  \begin{align*}
    H_n(s) \defeq
    \frac 1 n \sum_{i=1}^n \ind_{\lambda_i^2/d \leq s} .
  \end{align*}
  Then with probability one $H_n$ converges weakly to the c.d.f.\ $H$
  supported on $[\lambda_-, \lambda_+]$, with
  \begin{equation*}
    \lambda_+ \defeq \prn{1 + \frac{1}{\sqrt{\gamma}}}^2
    ~~~ \mbox{and} ~~~
    \lambda_- \defeq \prn{1 - \frac{1}{\sqrt{\gamma}}}^2,
  \end{equation*}
  and $H$ has density
  \begin{equation*} 
    dH(s) = \frac{\gamma}{2\pi} \frac{\sqrt{(\lambda_+ -s)(s - \lambda_-)}}{s} \ind_{s \in [\lambda_-, \lambda_+]} ds.
  \end{equation*}
\end{lemma}

\begin{lemma}[Bai-Yin law, \citet{BaiSi10}, Thm.~5.10]  \label{thm:BY-law}
  Let the conditions of Lemma~\ref{thm:MP-law} hold, and assume additionally
  that $\sup_{ij} \Ep[z_{ij}^4] < \infty$. Then the largest and smallest
  singular values
  $\lambda_1 = \lambda_1(Z)$ and $\lambda_n = \lambda_n(Z)$ of $Z$ satisfy
  \begin{equation*}
    \frac{\lambda_1^2}{d} \cas \lambda_+
    = \prn{1 + \frac{1}{\sqrt \gamma}}^2 \, , \qquad
    \frac{\lambda_n^2}{d} \cas \lambda_-
    = \prn{1 - \frac{1}{\sqrt \gamma}}^2.
  \end{equation*}
\end{lemma}

We also consider random matrices whose rows have non-identity covariance. In
these cases, we assume a deterministic sequence of symmetric positive
definite matrices $\Sigma := \Sigma(n) \in \real^{d \times d}$.
We let $t_1 \ge t_2 \ge \cdots \ge t_d > 0$ denote the eigenvalues of
$\Sigma$ and let $T_n$ denote the associated c.d.f.\
\begin{equation*}
  T_n(s) \defeq \frac{1}{d} 
  \sum_{i=1}^d \ind_{t_i \leq s},
\end{equation*}
assuming that $T_n$ converges weakly to some c.d.f.\ $T$ on $\R_+$.
With this, we can state a limiting law for the spectral distribution
of $\frac{1}{d} Z \Sigma Z^\top$. In the statement of the lemma,
we require the \emph{Stieltjes transform} of a measure. Letting
$\C_+ \defeq \{z \in \C \mid \textup{Im}(z) > 0\}$ be those elements of $\C$
with positive imaginary part,
recall that for a measure on $\R$ with c.d.f.\ $F$, the
Stieltjes transform of $m_F : \C_+ \to \C_+$ of $F$ is
\begin{align}
  \label{eq:stieltjes-transform}
  m_F(z) \defeq \int \frac{1}{s -z} dF(s).
\end{align}
Then we have the following

\begin{lemma}[Deformed Marchenko-Pastur law, \citet{Silverstein95}]
  \label{thm:MP-deformed-law}
  Let the conditions of Lemma~\ref{thm:MP-law} and those on
  the spectral distribution $T_n$ of $\Sigma$ above hold.
  Let $\frac{1}{d} Z\Sigma Z^\top$ have spectral distribution with c.d.f.\
  \begin{align*}
    G_n(s) := \frac{1}{n} \sum_{i=1}^n \ind_{\lambda_i^2/d \leq s}.
  \end{align*}
  Then with probability one, $G_n$ converges weakly to the c.d.f.\
  $G$ whose  Stieltjes transform $m_G$
  satisfies the fixed point equation
  \begin{equation*}
    m_G(z) = - \prn{
      z - \int \frac{\tau }{1 + \tau m_G(z)/\gamma} d T(\tau)
    }^{-1}.
  \end{equation*}
\end{lemma}
\noindent
Lemma~\ref{thm:MP-deformed-law} is slightly different from the result of
\citet[Thm.~1.1]{Silverstein95}, whose original theorem holds for the
empirical spectral distributions of $\frac{1}{n} Z \Sigma Z^\top$.
Lemma~\ref{thm:MP-deformed-law} follows from the change of variables
$n = \frac{d}{\gamma}(1 + o(1))$.

%% file: sections/proof-technical.tex
\section{Proofs of identities in Theorem~\ref{thm:cost-isotropic}}
\label{proof:thechnical}

\subsection{Proof of Lemma~\ref{lem:closed-form-errors}} \label{proof:closed-form-errors}

This is essentially trivial: by definition, we can write
\begin{align*}
  \pred_X(\est{\theta})
  & = \Ep_\theta \brk{\pred_{X, \theta} (\est{\theta})}
  = \Ep_{\theta, w} \brk{\norm{(AX-I)\theta + Aw}_{\Sigma}^2 \mid X} \nonumber \\
  & = \Tr \prn{\Ep_{\theta, w} \brk{\prn{(AX-I)\theta + Aw}^\top  \Sigma \prn{(AX-I)\theta + Aw} \mid X}} \nonumber \\
  & = \Tr \prn{\Ep_\theta \brk{\Sigma (AX-I) \theta \theta^\top (AX-I)^\top  \mid X} } + \sigma^2 \Tr \prn{A^\top \Sigma A} \nonumber \\
  & =  \frac{1}{d} \norm{\Sigma^{\frac{1}{2}} \prn{AX-I}}_F^2 + \sigma^2 \norm{\Sigma^{\frac{1}{2}} A}_F^2,
\end{align*}
where in the last line we use $\Ep[\theta \theta^\top] = I_d/d$. Similarly
\begin{align*}
  \train_X(\est{\theta}) & = \Ep_\theta \brk{\train_{X, \theta}(\est{\theta})}
  = \frac{1}{n} \Ep_{\theta, w} \brk{\norm{\prn{XA - I}\prn{X \theta + w}}_2^2 \mid X} \nonumber \\
	& = \frac{1}{n} \Tr \prn{\Ep_{\theta, w} \brk{\prn{X \theta + w}^\top \prn{XA - I}^\top\prn{XA - I}\prn{X \theta + w} \mid X} } \nonumber \\
	& = \frac{1}{n} \Tr \prn{(XA - I)X \theta \theta^{\top} X^{\top}(XA - I)^\top} + \frac{\sigma^2}{n} \Tr \prn{(XA-I)(XA-I)^\top} \nonumber \\
	& = \frac{1}{nd} \norm{XAX-X}_F^2 + \frac{\sigma^2}{n} \norm{XA-I}_F^2 \, .
\end{align*}

\subsection{Proof of Lemma~\ref{lem:strong-duality}} \label{proof:strong-duality}

While problem~\eqref{eq:core-non-convex} is non-convex, it consists of a
quadratic objective and quadratic constraint, and taking $A \to \infty$
shows that there certainly exist feasible points in the interior of the set
of $A$ satisfying $\mc{T}(A; \Sigma) \ge \epsilon^2$. Thus, strong duality
holds~\citep[Appendix~B.1]{BoydVa04}. We therefore consider the Lagrangian
dual problem, introducing the dual multplier $\lambda \ge 0$ for
the constraint and writing the Lagrangian
\begin{align*}
  \mathcal{L}(A, \lambda) & = \mathcal{P}(A; \Sigma) + \lambda(\epsilon^2 - \mathcal{T}(A; \Sigma)) \nonumber \\
	& =  \frac{1}{d} \norm{\Sigma^{\frac{1}{2}} \prn{AX-I}}_F^2 + \sigma^2 \norm{\Sigma^{\frac{1}{2}}A}_F^2 - \lambda\prn{\frac{1}{nd} \norm{XAX-X}_F^2 + \frac{\sigma^2}{n} \norm{XA-I}_F^2 } + \lambda \epsilon^2.
\end{align*}

Using $\mc{L}$, we begin by demonstrating the first claim of
the lemma, that is, that if $\Sigma - \frac{\lambda}{n} X^T X \not\succ 0$,
then we have $\inf_A \mc{L}(A, \lambda) = -\infty$.
To see this, first 
let $V \in \real^{d \times r}$ be an orthogonal basis for $X$'s row space
and $V^\perp$ its orthogonal complement. Then $X V^\perp = 0$ and $\Sigma
- \frac{ \lambda}{n} X^\top X$ failing to be positive definite is equivalent to
\begin{align*}
  \begin{bmatrix} V & V^\perp \end{bmatrix}^\top \prn{\Sigma - \frac{ \lambda}{n} X^\top X} \begin{bmatrix} V & V^\perp \end{bmatrix} = \begin{bmatrix} V^\top  \prn{\Sigma - \frac{ \lambda}{n} X^\top X} V  & 0
    \\ 0 & (V^\perp)^\top \Sigma V^\perp \end{bmatrix} 
\end{align*}
failing to be positive definite. Then as $(V^\perp)^\top \Sigma V^\perp
\succ 0$ by assumption that $\Sigma \succ 0$, it must thus be the case that
$V^\top (\Sigma - \frac{ \lambda}{n} X^\top X) V \not\succ 0$. We leverage
this indefiniteness to observe that, as $V$ spans the row space of $X$, 
there exists a unit vector $\nu \in \real^d$, $\norm{\nu} = 1$, and
vector $\mu \in \R^n$ satisfying $\nu = X^\top \mu \in \real^d$ and
\begin{align}
  \alpha \defeq \nu^\top \prn{\Sigma - \frac{ \lambda}{n} X^\top X } \nu
  \le 0. \label{eq:strong-duality-mid-1}
\end{align}

To show that the non-positivity~\eqref{eq:strong-duality-mid-1} entails
$\inf_A \mc{L}(A, \lambda) = -\infty$ requires a few additional steps. We
detour by taking the gradient of the Lagrangian with respect to $A$
(this will be useful later), 
\begin{align}
  \lefteqn{\frac{\partial }{\partial A} \mathcal{L}(A,  \lambda)} \nonumber \\
  & = \frac{1}{d} \prn{\Sigma A XX^\top - \Sigma X^\top} + \sigma^2 \Sigma A - \frac{ \lambda}{n} \brc{\frac{1}{d} \prn{X^\top X A XX^\top - X^\top XX^\top} + \sigma^2 \prn{X^\top X A - X^\top}}
  \nonumber \\
  & = \frac{1}{d} \prn{\Sigma - \frac{ \lambda}{n} X^\top X} AXX^\top + \sigma^2 \prn{\Sigma - \frac{ \lambda}{n} X^\top X} A - \frac{1}{d} \prn{\Sigma - \frac{ \lambda}{n} X^\top X - \frac{ \lambda d\sigma^2}{n} I } X^\top
  \nonumber \\
  & = \frac{1}{d} \prn{\Sigma - \frac{ \lambda}{n} X^\top X} A \prn{XX^\top + d \sigma^2 I} - \frac{1}{d} \prn{\Sigma - \frac{ \lambda}{n} X^\top X - \frac{ \lambda d\sigma^2}{n} I} X^\top .
  \label{eqn:lagrangian-derivative}
\end{align}
Using the $\mu$ defining $\nu = X^\top \mu$ in
Eq.~\eqref{eq:strong-duality-mid-1}, let $t \in \R$ be unspecified and take
$A = t \nu \mu^\top$. Define the function $L(t) = \mathcal{L}(t \nu
\mu^\top, \lambda)$, for which we have
\begin{align*}
  \frac{d}{dt} L(t) & =
  \Tr \prn{\frac{\partial }{\partial A} \mc{L}( t \nu \mu^\top,  \lambda)
    (\nu \mu^\top)^\top} \nonumber \\
  & = \frac{t}{d} \nu^\top \prn{\Sigma - \frac{ \lambda}{n} X^\top X} \nu
  \mu^\top \prn{XX^\top + d \sigma^2 I} \mu -
  \frac{1}{d} \nu^\top \prn{\Sigma - \frac{ \lambda}{n} X^\top X - \frac{ \lambda d\sigma^2}{n} I} X^\top \mu
  \nonumber \\
  & \stackrel{(i)}{=}
  \frac{t}{d} \alpha \cdot (\ltwo{\nu}^2 + d \sigma^2 \ltwo{\mu}^2)
  - \frac{1}{d} \nu^\top \prn{\Sigma - \frac{ \lambda}{n} X^\top X - \frac{ \lambda d\sigma^2}{n} I} \nu \nonumber \\
  & \stackrel{(ii)}{=} \frac{t \alpha}{d} \cdot \prn{1 + d \sigma^2 \ltwo{\mu}^2}
  - \frac{\alpha}{d} + \frac{\lambda \sigma^2}{n}
\end{align*}
where step $(i)$ substitutes the definition~\eqref{eq:strong-duality-mid-1} of
$\alpha$ and that $X^\top \mu = \nu$, while step $(ii)$ similarly uses
the definition of $\alpha$ and that $\ltwo{\nu} = 1$ by assumption.
We consider two cases: if $\alpha < 0$, then taking
$t \to \infty$ yields $L'(t) \to -\infty$, so that $L(t) \to -\infty$
and $\inf_A \mc{L}(A, \lambda) = -\infty$. If $\alpha = 0$, then
$L'(t) = \frac{\lambda \sigma^2}{n} > 0$, and so taking $t \to -\infty$
yields $\mc{L}(A, \lambda) \to -\infty$ as well. As such, the optimal
$\lambda \ge 0$ must satisfy $\Sigma - \frac{\lambda}{n} X^\top X \succ 0$,
as we desired to show.

Having verified that $\Sigma - \frac{\lambda}{n} X^\top X \succ 0$, we can
use the derivative~\eqref{eqn:lagrangian-derivative} and solve for the $A$
satisfying the stationary condition $\frac{\partial }{\partial A}
\mathcal{L}(A, \lambda) = 0$, obtaining
\begin{equation*}
  \frac{1}{d} \prn{\Sigma - \frac{ \lambda}{n} X^\top X} A \prn{XX^\top + d \sigma^2 I} - \frac{1}{d} \prn{\Sigma - \frac{ \lambda}{n} X^\top X - \frac{ \lambda d\sigma^2}{n} I } X^\top = 0.
\end{equation*}
Solving this equation yields
\begin{align*}
  A & =  \prn{\Sigma - \frac{ \lambda}{n} X^\top X}^{-1}  \prn{\Sigma - \frac{ \lambda}{n} X^\top X - \frac{ \lambda d\sigma^2}{n} I} X^\top \prn{XX^\top + d \sigma^2 I}^{-1} \nonumber \\
	& = \prn{I - \frac{ \lambda d \sigma^2}{n} \prn{\Sigma - \frac{ \lambda}{n} X^\top X}^{-1}} X^\top (XX^\top + d \sigma^2 I)^{-1} \\
	& = \prn{I - \frac{ \lambda d \sigma^2}{n} \prn{\Sigma - \frac{ \lambda}{n} X^\top X}^{-1}}  (X^\top X + d \sigma^2 I)^{-1} X^\top \, . 
\end{align*}
In the last equation we use the matrix identity $ X^\top (XX^\top + d
\sigma^2 I)^{-1} = (X^\top X + d \sigma^2 I)^{-1} X^\top $, which follows
directly via the SVD of $X$. We complete the proof by identifying
$\rho_n := \frac{\lambda n}{d}$.

\subsection{Proof of Lemma~\ref{lem:rho-error-forms}} \label{proof:rho-error-forms}

The proof is essentially pure calculations. For reference, we divide the
proof into three parts.
\begin{enumerate}[label=\Roman*.]
\item We compute formulas for  $A(\rho; \Sigma)X - I$ and $XA(\rho; \Sigma)- I$.
\item Derive the expansion for $\mathcal{P}(A(\rho, \Sigma); \Sigma) -
  \mathcal{P}(A(0, \Sigma); \Sigma)$.
\item Derive the expansion for for $\mathcal{T}(A(\rho, \Sigma); \Sigma) $. 
\end{enumerate}
Throughout we write $A(\rho) = A(\rho; \Sigma)$ for simplicity.

\paragraph{Part I: Computing $A(\rho)X - I$ and $XA(\rho)- I$. }
We first substitute expression~\eqref{eq:A-1} for $A(\rho)$ into the
difference $A(\rho)X-I$ to obtain
\begin{align}
  \lefteqn{A(\rho)X - I} \nonumber \\
  & =  \prn{I - \rho \sigma^2 \prn{\Sigma - \frac{\rho}{d} X^\top X}^{-1}}  (X^\top X + d \sigma^2 I)^{-1} X^\top X - I \nonumber \\
	& = - \rho \sigma^2 \prn{\Sigma - \frac{\rho}{d} X^\top X}^{-1}  (X^\top X + d \sigma^2 I)^{-1} X^\top X +  (X^\top X + d \sigma^2 I)^{-1} \prn{X^\top X  - X^\top X - d \sigma^2 I} \nonumber  \\
	& \stackrel{(i)}{=} - \rho \sigma^2 \prn{\Sigma - \frac{\rho}{d} X^\top X}^{-1} X^\top X (X^\top X + d \sigma^2 I)^{-1}  - d\sigma^2 (X^\top X + d \sigma^2 I)^{-1} \nonumber \\
	& = \brc{- \rho \sigma^2 \prn{\Sigma - \frac{\rho}{d} X^\top X}^{-1} X^\top X - d\sigma^2 \prn{\Sigma - \frac{\rho}{d} X^\top X}^{-1}  \prn{\Sigma - \frac{\rho}{d} X^\top X}} (X^\top X + d \sigma^2 I)^{-1} \nonumber \\
	& = - \sigma^2 \prn{\Sigma - \frac{\rho}{d} X^\top X}^{-1} \brc{\rho X^\top X + d \prn{\Sigma - \frac \rho d X^\top X}}  (X^\top X + d \sigma^2 I)^{-1} \nonumber \\
	& = - d \sigma^2 \prn{\Sigma - \frac{\rho}{d} X^\top X}^{-1} \Sigma  (X^\top X + d \sigma^2 I)^{-1} \, , \label{eq:AX-minus-I}
\end{align}
where in step~$(i)$ we use that $X^\top X$ and $(X^\top X + d\sigma^2 I)^{-1}$
commute. Similarly, we can compute $XA(\rho)-I$ by using the alternative
formulation~\eqref{eq:A-2} for $A(\rho)$, substituting to obtain
\begin{align*}
  \lefteqn{XA(\rho) - I} \\
  & = X \prn{ I - \rho \sigma^2 \prn{\Sigma - \frac{\rho}{d} X^\top X}^{-1}}  X^\top (X X^\top + d \sigma^2 I)^{-1}   - I  \nonumber  \\
	& = - \rho \sigma^2 X \prn{\Sigma - \frac{\rho}{d} X^\top X}^{-1} X^\top (X X^\top + d \sigma^2 I)^{-1} + \prn{ X  X^\top - X X^\top - d \sigma^2 I } (X X^\top + d \sigma^2 I)^{-1}  \nonumber \\
	& = - \rho \sigma^2 X \prn{\Sigma - \frac{\rho}{d} X^\top X}^{-1} X^\top (X X^\top + d \sigma^2 I)^{-1} - d\sigma^2 (X X^\top + d \sigma^2 I)^{-1}  \nonumber \\
	& = - d\sigma^2 \brc{\frac \rho d X \prn{\Sigma - \frac \rho d X^\top X}^{-1} X^\top + I} \prn{XX^\top + d\sigma^2 I}^{-1} \, .
\end{align*}
As $X$ is wide and $XX^\top$ is non-singular by assumption, $\lim_{\lambda
  \downarrow 0} X(X^\top X+ \lambda I)^{-1} X^\top = I$ and therefore
\begin{align}
  \lefteqn{XA(\rho) - I} \nonumber \\
  & = - d\sigma^2 \brc{\frac \rho d X \prn{\Sigma - \frac \rho d X^\top X}^{-1} X^\top + \lim_{\lambda \downarrow 0} X(X^\top X+ \lambda I)^{-1} X^\top } \prn{XX^\top + d\sigma^2 I}^{-1} \nonumber \\
  & = - \lim_{\lambda \downarrow 0}
  d\sigma^2 X   \brc{ \prn{\frac d \rho \Sigma - X^\top X}^{-1} + \prn{X^\top X + \lambda I}^{-1} } \cdot X^\top \prn{XX^\top + d\sigma^2 I}^{-1} \nonumber \\
  & = - \lim_{\lambda \downarrow 0}
  d\sigma^2 X \cdot   \brc{ \prn{\frac d \rho \Sigma - X^\top X}^{-1} \prn{\lambda I + \frac d \rho \Sigma} \prn{X^\top X + \lambda I}^{-1}  } X^\top \cdot \prn{XX^\top + d\sigma^2 I}^{-1} \nonumber \\
  & \stackrel{(i)}{=} - \lim_{\lambda \downarrow 0} d\sigma^2 X \cdot  \brc{ \prn{\frac d \rho \Sigma - X^\top X}^{-1} \prn{\lambda I + \frac d \rho \Sigma} X^\top \prn{X X^\top + \lambda I}^{-1}  } \cdot \prn{XX^\top + d\sigma^2 I}^{-1} \nonumber \\
  & = - d\sigma^2 X \prn{\Sigma - \frac \rho d X^\top X}^{-1} \Sigma X^\top \prn{XX^\top}^{-1} \prn{XX^\top + d\sigma^2 I}^{-1} \, ,\label{eq:XA-minus-I}
\end{align}
where in step $(i)$ we use that $\prn{X^\top X + \lambda I}^{-1} X^\top =  X^\top \prn{X X^\top + \lambda I}^{-1} $.

\paragraph{Part II: Computing $\mathcal{P}(A(\rho, \Sigma); \Sigma) - \mathcal{P}(A(0, \Sigma); \Sigma)$.}
As
\begin{align}
  &\mathcal{P}(A(\rho, \Sigma); \Sigma) - \mathcal{P}(A(0, \Sigma); \Sigma)
  \label{eq:p-error-growth-mid-1} \\
	& = \frac{1}{d} \norm{\Sigma^{\frac{1}{2}} \prn{A(\rho)X-I}}_F^2 + \sigma^2 \norm{\Sigma^{\frac{1}{2}}A(\rho)}_F^2 - \frac{1}{d} \norm{\Sigma^{\frac{1}{2}} \prn{A(0)X-I}}_F^2 - \sigma^2 \norm{\Sigma^{\frac{1}{2}}A(0)}_F^2 \nonumber \\
	& =  \underbrace{\frac{1}{d}\prn{\norm{\Sigma^{\frac{1}{2}} \prn{A(\rho)X-I}}_F^2 - \norm{\Sigma^{\frac{1}{2}} \prn{A(0)X-I}}_F^2}}_{\mathrm{(I)}} + \underbrace{\sigma^2  \prn{\norm{\Sigma^{\frac{1}{2}}A(\rho)}_F^2 - \norm{\Sigma^{\frac{1}{2}}A(0)}_F^2}}_{\mathrm{(II)}}, \nonumber
\end{align}
we compute terms (I) and (II) separately. For (I) we substitute in the
explicit form~\eqref{eq:AX-minus-I} of $A(\rho)X-I$  to obtain
\begin{align*}
  \mathrm{(I)} & =  d \sigma^4 \Tr \prn{ \Sigma \prn{\Sigma - \frac{\rho}{d} X^\top X}^{-1} \Sigma \prn{X^\top X + d\sigma^2 I}^{-2}  \Sigma  \prn{\Sigma - \frac{\rho}{d} X^\top X}^{-1}} -  d \sigma^4 \Tr \prn{ \Sigma \prn{X^\top X + d\sigma^2 I}^{-2} }.
\end{align*}
We then use the identity
\begin{align*}
  \prn{\Sigma - \frac \rho d X^\top X}^{-1} \Sigma = I + \prn{\Sigma - \frac \rho d X^\top X}^{-1} \cdot \frac \rho d X^\top X
\end{align*}
to obtain further that
\begin{align}
	\mathrm{(I)} & = d \sigma^4 \Tr \prn{ \Sigma \prn{I + \prn{\Sigma - \frac \rho d X^\top X}^{-1} \cdot \frac \rho d X^\top X} \prn{X^\top X + d\sigma^2 I}^{-2} \prn{I + \frac \rho d X^\top X \cdot \prn{\Sigma - \frac \rho d X^\top X}^{-1} }} \nonumber \\
	& \qquad - d \sigma^4 \Tr \prn{ \Sigma \prn{X^\top X + d\sigma^2 I}^{-2} } \nonumber \\
	& =  d \sigma^4 \Tr \prn{ \Sigma  \prn{\Sigma - \frac \rho d X^\top X}^{-1} \cdot \frac \rho d X^\top X \prn{X^\top X + d\sigma^2 I}^{-2}} \nonumber \\
	& \qquad  + d \sigma^4 \Tr \prn{ \Sigma \prn{X^\top X + d\sigma^2 I}^{-2} \frac \rho d X^\top X \cdot \prn{\Sigma - \frac \rho d X^\top X}^{-1}  }  \nonumber \\
	& \qquad + d \sigma^4 \Tr \prn{ \Sigma \prn{\Sigma - \frac \rho d X^\top X}^{-1} \cdot \frac \rho d X^\top X \prn{X^\top X + d\sigma^2 I}^{-2} \frac \rho d X^\top X \cdot \prn{\Sigma - \frac \rho d X^\top X}^{-1}  }  \nonumber \\
	& =  \rho \sigma^4 \Tr \prn{ \Sigma  \prn{\Sigma - \frac \rho d X^\top X}^{-1} X^\top X \prn{X^\top X + d\sigma^2 I}^{-2}}
        \label{eq:p-error-growth-mid-3} \\
	& \qquad  + \rho \sigma^4 \Tr \prn{ \Sigma \prn{X^\top X + d\sigma^2 I}^{-2}  X^\top X   \prn{\Sigma - \frac \rho d X^\top X}^{-1}  }  \nonumber \\
	& \qquad + \frac{\rho^2 \sigma^4}{d} \Tr \prn{ \Sigma \prn{\Sigma - \frac \rho d X^\top X}^{-1} X^\top X \prn{X^\top X + d\sigma^2 I}^{-2}  X^\top X \prn{\Sigma - \frac \rho d X^\top X}^{-1}  }. \nonumber
\end{align}
For term (II), we substitute in formula~\eqref{eq:A-2} for $A(\rho)$ and use
that $X^\top X$ and $(X^\top X + d\sigma^2 I)^{-1}$ commute, yielding
that
\begin{align*}
  \mathrm{(II)} & =  \sigma^2 \Tr \prn{\Sigma \prn{I - \rho \sigma^2 \prn{\Sigma - \frac{\rho}{d} X^\top X}^{-1}}  X^\top X (X^\top X + d \sigma^2 I)^{-2}   \prn{I - \rho \sigma^2 \prn{\Sigma - \frac{\rho}{d} X^\top X}^{-1}}} \nonumber \\
  & \qquad - \sigma^2  \Tr \prn{\Sigma X^\top X (X^\top X + d \sigma^2 I)^{-2}} \nonumber \\
  & = -\rho \sigma^4 \Tr \prn{ \Sigma  \prn{\Sigma - \frac \rho d X^\top X}^{-1} X^\top X \prn{X^\top X + d\sigma^2 I}^{-2}} \nonumber \\
  & \qquad  - \rho \sigma^4 \Tr \prn{ \Sigma \prn{X^\top X + d\sigma^2 I}^{-2}  X^\top X   \prn{\Sigma - \frac \rho d X^\top X}^{-1}  }  \nonumber \\
  & \qquad + \rho^2 \sigma^6 \Tr \prn{ \Sigma \prn{\Sigma - \frac \rho d X^\top X}^{-1} X^\top X \prn{X^\top X + d\sigma^2 I}^{-2} \prn{\Sigma - \frac \rho d X^\top X}^{-1}  } .
\end{align*}
Substituting the equality~\eqref{eq:p-error-growth-mid-3} for term (I) and
the above identity for term (II) back into
the expansion~\eqref{eq:p-error-growth-mid-1} of
$\mc{P}(A(\rho, \Sigma); \Sigma) - \mc{P}(A(0,\Sigma);\Sigma)$,
we get our desired expansion:
\begin{align*}
  \lefteqn{\mathcal{P}(A(\rho, \Sigma); \Sigma) - \mathcal{P}(A(0, \Sigma); \Sigma)} \\
	& =\frac{\rho^2 \sigma^4}{d} \Tr \prn{ \Sigma \prn{\Sigma - \frac \rho d X^\top X}^{-1} X^\top X \prn{X^\top X + d\sigma^2 I}^{-2}  X^\top X \prn{\Sigma - \frac \rho d X^\top X}^{-1}  } \nonumber \\
	& \qquad + \rho^2 \sigma^6 \Tr \prn{ \Sigma \prn{\Sigma - \frac \rho d X^\top X}^{-1} X^\top X \prn{X^\top X + d\sigma^2 I}^{-2} \prn{\Sigma - \frac \rho d X^\top X}^{-1}  } \nonumber \\
	& = \frac{\rho^2 \sigma^4}{d} \Tr \prn{ \Sigma \prn{\Sigma - \frac \rho d X^\top X}^{-1} X^\top X \prn{X^\top X + d\sigma^2 I}^{-2}  \prn{X^\top X + d\sigma^2 I} \prn{\Sigma - \frac \rho d X^\top X}^{-1}  } \nonumber \\
	& = \frac{\rho^2  \sigma^4}{d} \Tr \prn{\prn{\Sigma - \frac \rho d X^\top X}^{-1}\Sigma \prn{\Sigma - \frac \rho d X^\top X}^{-1} X^\top X  \prn{X^\top X + d \sigma^2 I}^{-1} } \, .
\end{align*}

\paragraph{Part III: Computing $\mathcal{T}(A(\rho, \Sigma); \Sigma) $.}
Leveraging the expansion
\begin{align*}
	\mathcal{T}(A(\rho, \Sigma); \Sigma) & = \frac{1}{nd} \norm{XA(\rho)X-X}_F^2 + \frac{\sigma^2}{n} \norm{XA(\rho)-I}_F^2 \nonumber \\
	& =  \frac{1}{nd} \Tr \prn{(XA(\rho)-I)XX^\top (XA(\rho)-I)^\top} +  \frac{\sigma^2}{n} \Tr \prn{(XA(\rho)-I)(XA(\rho)-I)^\top} \nonumber \\
	& = \frac{1}{nd} \Tr \prn{(XA(\rho)-I) \prn{XX^\top + d\sigma^2 I} (XA(\rho)-I)^\top} \, ,
\end{align*}
we can substitute the expression~\eqref{eq:XA-minus-I} for $X A(\rho) - I$
to obtain
\begin{align*} 
  \lefteqn{\mathcal{T}(A(\rho, \Sigma); \Sigma) =}  \nonumber \\
  & \frac{d \sigma^4}{n} \Tr \prn{  X \prn{\Sigma - \frac \rho d X^\top X}^{-1} \Sigma X^\top \prn{XX^\top}^{-1} \prn{XX^\top + d\sigma^2 I}^{-1} \prn{XX^\top}^{-1}X \Sigma  \prn{\Sigma - \frac \rho d X^\top X}^{-1} X^\top }.
\end{align*}
Leveraging the identity $X^\top (XX^\top + d \sigma^2 I)^{-1} = (X^\top X + d\sigma^2 I)^{-1} X^\top$ and that $(XX^\top)^{-1}$ and $(XX^\top + \lambda I)^{-1}$
commute, we have
\begin{align*}
  X^\top (XX^\top)^{-1} (XX^\top + d\sigma^2 I)^{-1}
  (XX^\top)^{-1}X
  & = X^\top (XX^\top)^{-2} X (X^\top X + d \sigma^2 I)^{-1}  \\
  & = (X^\top X)^\dagger (X^\top X + d\sigma^2 I)^{-1}.
\end{align*}
Substituting this into the preceding display gives
\begin{align*}
  \lefteqn{\mc{T}(A(\rho, \Sigma); \Sigma)} \\
  & = \frac{d \sigma^4}{n}
  \Tr \prn{X \prn{\Sigma - \frac{\rho}{d} X^\top X}^{-1}
    \Sigma (X^\top X)^\dagger (X^\top X + d \sigma^2 I)^{-1}
    \Sigma \prn{\Sigma - \frac{\rho}{d} X^\top X}^{-1} X^\top} \\
  & = \frac{d \sigma^4}{n} \Tr \prn{\Sigma \prn{\Sigma - \frac \rho d X^\top X}^{-1} X^\top X \prn{\Sigma - \frac \rho d X^\top X}^{-1} \Sigma \prn{X^\top X}^\dagger \prn{X^\top X + d\sigma^2 I}^{-1} } 
\end{align*}
by the cyclic property of the trace, as desired.

%% file: sections/proof-isotropic-additional.tex


\subsection{Proof of Lemma~\ref{lem:cost-isotropic-mid-1}} \label{proof:cost-isotropic-mid-1}
By \citet[Lemma~3.11]{BaiSi10} we can exactly compute
\begin{align*}
	\lim_{y \to 0^+} m_H(-\sigma^2 + i y) & = \frac{1 - 1/\gamma + \sigma^2 - \sqrt{\prn{1 + 1 / \gamma + \sigma^2 }^2 - 4/\gamma}}{- 2 \sigma^2/\gamma} \nonumber \\
	& =  \frac{ \sqrt{\prn{1 - 1 / \gamma + \sigma^2 }^2 + 4 \sigma^2/\gamma} - \prn{1 - 1 / \gamma + \sigma^2 }}{2 \sigma^2/\gamma} \nonumber \\
	& = \frac{2\sigma^2 / \gamma + o(\sigma^2/\gamma)}{2 \sigma^2/\gamma \cdot \prn{1 - 1/\gamma + \sigma^2}} \, ,
\end{align*}
completing the proof.

\subsection{Proof of Lemma~\ref{lem:cost-isotropic-mid-2}}
\label{proof:cost-isotropic-mid-2}

As the Bai-Yin law (Lemma~\ref{thm:BY-law}) guarantees the convergence of
the smallest eigenvalue of $\frac{1}{n} XX^\top$ and $XX^\top$ is eventually
non-singular, we can invoke the identities on the prediction and training
error in Lemma~\ref{lem:rho-error-forms}. Therefore
\begin{align*}
  \mathcal{P}(A(\rho, I); I) - \mathcal{P}(A(0, I); I) & = \frac{\rho^2  \sigma^4}{d} \Tr \prn{\prn{I - \frac{\rho}{d} X^\top X}^{-2}X^\top X  \prn{X^\top X + d \sigma^2 I}^{-1} } \nonumber \\
  & = \frac{\rho^2  \sigma^4}{d /n}  \cdot \frac{1}{n} \sum_{i=1}^n \frac{1}{\prn{1 - \rho\lambda_i^2/d}^2} \cdot \frac{\lambda_i^2}{d} \cdot \frac{1}{\lambda_i^2/d + \sigma^2}  \nonumber \\
  & = \frac{\rho^2  }{d /n} \int\frac{\sigma^4 s}{(1-\rho s)^2 (s + \sigma^2)} dH_n(s).
\end{align*}
By the assumption that $\rho < \lambda_+^{-1}$, the Bai-Yin law
(Lemma~\ref{thm:BY-law}) guarantees that $I - \frac{\rho}{d} XX^\top$ is
eventually positive definite and with probability one $\lambda_1^2/ d \to
\lambda_+$.  The function $s \mapsto \frac{\sigma^4 s}{(1 - \rho s)^2(s +
  \sigma^2)}$ is thus eventually bounded on the support of
$H_n$. Applying the  Marchenko-Pastur law, we
deduce
\begin{equation*}
  \lim_{n \to \infty} \prn{\mathcal{P}(A(\rho, I); I) - \mathcal{P}(A(0, I); I)}  = \frac{\rho^2  }{\gamma} \int\frac{\sigma^4 s}{(1-\rho s)^2 (s + \sigma^2)} dH(s).
\end{equation*}

For the second limit in Lemma~\ref{lem:cost-isotropic-mid-2}, we can
again leverage $\Sigma = I$ in Lemma~\ref{lem:rho-error-forms}
to compute
\begin{align*}
	\mathcal{T}(A(\rho, I); I) & =  \frac{d \sigma^4}{n} \Tr \prn{\prn{I - \frac{\rho}{d} X^\top X}^{-1} X^\top X \prn{I - \frac{\rho}{d} X^\top X}^{-1}  \prn{X^\top X}^\dagger \prn{X^\top X + d\sigma^2 I}^{-1} } \nonumber \\
	& =  \frac{ \sigma^4}{n} \Tr \prn{\prn{I - \frac{\rho}{d} X^\top X}^{-1} \frac{X^\top X}{d} \prn{I - \frac{\rho}{d} X^\top X}^{-1}  \prn{\frac{X^\top X}{d}}^\dagger \prn{\frac{X^\top X}{d} + \sigma^2 I}^{-1} }  \nonumber \\
	& = \sigma^4  \cdot \frac{1}{n} \sum_{i=1}^n \frac{1}{1 - \rho\lambda_i^2/d} \cdot \frac{\lambda_i^2}{d} \cdot \frac{1}{1 - \rho\lambda_i^2/d} \cdot  \frac{1}{\lambda_i^2/d} \cdot \frac{1}{\lambda_i^2/d + \sigma^2}  \nonumber \\
	& = \int \frac{\sigma^4 }{(1 - \rho s)^2(s + \sigma^2)} dH_n(s).
\end{align*}
Applying the Marchenko-Pastur law gives the desired limit.

%% file: sections/proof-isotropic-interpolator.tex
\section{Proof of Theorem~\ref{thm:cost-interpolator-isotropic}}
\label{proof:cost-interpolator-isotropic}
  
We only need to prove under Assumption~\ref{assp:linear} thanks to Theorem~\ref{thm:strong-duality-hilbert}. First, we recall our standard notation that $X$ has singular values $\lambda_1 \geq
\lambda_2 \geq \cdots \geq \lambda_n \ge 0$ and empirical spectral
c.d.f.\ $H_n(s) = \frac{1}{n} \sum_{i = 1}^n \ind_{\lambda_i^2/d \le s}$.
We first prove (most of) part~\ref{item:relative-costs-ols} of the theorem,
which we state as a lemma. It is immediate by the
definitions~\eqref{eq:def-cost} and~\eqref{eq:def-cost-interpolator} of
$\cost$ and $\wb{\cost}$ that $\cost_X(\epsilon) - \wb{\cost}_X(\epsilon) =
\pred_X (\thetaols) - \pred_X (\est{\theta}(0))$, so we focus on the latter
quantity.
\begin{lemma}
  \label{lemma:cost-difference-ols}
  With probability 1
  \begin{equation*}
    \lim_{n \to \infty} \prn{\pred_X \prn{\thetaols} - \pred_X \prn{\est{\theta}(0)}}
    = \frac{\sigma^2}{\gamma} \prn{\int \frac{1}{s} dH(s)
      - \int \frac{1}{s + \sigma^2} dH(s)}.
  \end{equation*}
\end{lemma}
\begin{proof}
  By the Bai-Yin law (Lemma~\ref{thm:BY-law}) we may assume that
  $XX^\top \succ 0$, as this eventually holds with probability 1.
  Let $\est{\theta}(0) = A(0, I) y$ for
  $A(0, I) = (X^\top X + d \sigma^2 I)^{-1} X^\top$ be the optimal
  unconstrained estimator (recall Lemma~\ref{lem:strong-duality}) and
  $\thetaols = \Aols y$ for $\Aols = X^\top(X X^\top)^{-1} = X^\dagger$.
  Then
  \begin{equation}
    \label{eqn:cost-difference-ols}
    \pred_X \prn{\est{\theta}_{\mathsf{OLS}}} - \pred_X \prn{\est{\theta}(0)}
    = \mathcal{P}(\Aols; I)  - \mathcal{P}(A(0, I); I).
  \end{equation}
  We expand each of the
  prediction errors above in turn.

  For the first, we have the identity
  \begin{align*}
    \lefteqn{\mc{P}(\Aols; I)
      = \frac{1}{d}
      \norm{\Aols X-I}_F^2 + \sigma^2 \norm{\Aols}_F^2} \\
    & = \frac{1}{d}
    \Tr \prn{\prn{X^\top (X X^\top)^{-1} X - I_d}^2}
    + \sigma^2 \Tr \prn{X^\top (XX^\top)^{-2} X}
    = \frac{d - n}{d} + \sigma^2 \Tr\prn{(XX^\top)^{-1}},
  \end{align*}
  where we have used that $X^\top (XX^\top)^{-1} X - I_d$ is a projection matrix
  of rank $d - n$.
  For the second,
  \begin{align*}
    \lefteqn{\mc{P}(A(0, I); I)
      = 
      \frac{1}{d} \norm{A(0, I)X-I}_F^2 + \sigma^2 \norm{A(0, I)}_F^2} \\
    & =
    \frac{1}{d} \Tr \prn{\prn{X^\top (X X^\top + d\sigma^2 I)^{-1} X - I}^2}
    + \sigma^2 \Tr \prn{X^\top \prn{XX^\top + d\sigma^2 I}^{-2} X} \\
    & \stackrel{(i)}{=} 1 + \frac{1}{d} \Tr\prn{
      (X X^\top)^2 (XX^\top + d \sigma^2 I)^{-2}
      - 2 XX^\top (XX^\top + d \sigma^2 I)^{-1}}
    + \sigma^2 \Tr\prn{XX^\top (XX^\top + d \sigma^2 I)^{-2}} \\
    & = 1 + \frac{1}{d}
    \Tr\prn{XX^\top \prn{XX^\top - 2 (XX^\top + d \sigma^2 I)
        + d \sigma^2 I} \prn{XX^\top + d \sigma^2 I}^{-2}} \\
    & = 1 + \frac{1}{d} \Tr\prn{XX^\top \prn{XX^\top + d \sigma^2 I}^{-1}},
  \end{align*}
  where in step~$(i)$ we use that $XX^\top$ and $(XX^\top + d \sigma^2 I)^{-1}$
  commute and the cyclic property of the trace.
  Substituting these equalities into expression~\eqref{eqn:cost-difference-ols}
  yields
  \begin{equation*}
    \pred_X \prn{\est{\theta}_{\mathsf{OLS}}} - \pred_X \prn{\est{\theta}(0)}
    = -\frac{n}{d}
    + \sigma^2 \Tr \prn{(XX^\top)^{-1}} + \frac{1}{d} \Tr \prn{XX^\top \prn{XX^\top + d\sigma^2 I}^{-1}}.
  \end{equation*}

  From this point, we expand the traces in terms of the empirical
  spectral distributions $H_n$, so multiplying and dividing
  $XX^\top$ by $d$ and normalizing the traces by $n$, we obtain
  \begin{align*}
    \pred_X \prn{\est{\theta}_{\mathsf{OLS}}} - \pred_X \prn{\est{\theta}(0)}
    & = -\frac{n}{d}
    + \frac{\sigma^2 n}{d} \int \frac{1}{s} dH_n(s)
    + \frac{n}{d} \int \frac{s}{s + \sigma^2} dH_n(s).
  \end{align*}
  We may apply the Bai-Yin law (Lemma~\ref{thm:BY-law}) and the
  Marchenko-Pastur law (Lemma~\ref{thm:MP-law}), so $\lambda_{\min}(XX^\top
  / d)$ converges with probability 1, and thus almost surely
  \begin{align*}
    \lim_{n \to \infty}
    \prn{
      \pred_X \prn{\est{\theta}_{\mathsf{OLS}}} - \pred_X \prn{\est{\theta}(0)}}
    & = -\frac{1}{\gamma}
    + \frac{\sigma^2}{\gamma} \int \frac{1}{s} dH(s)
    + \frac{1}{\gamma} \int \frac{s}{s + \sigma^2} dH(s).
  \end{align*}
  An algebraic manipulation gives the lemma.
\end{proof}

Noting that $\frac{1}{s} - \frac{1}{s + \sigma^2}
= \frac{\sigma^2}{s(s + \sigma^2)}$ gives the first equality of
part~\ref{item:relative-costs-ols} of the theorem.
We divide the remainder of the proof into two parts. In the first, we
perform an asymptotic expansion of the integral in
Lemma~\ref{lemma:cost-difference-ols} to finalize
part~\ref{item:relative-costs-ols}. In the second, we prove
part~\ref{item:interp-threshold-ols}, including the existence
of the threshold $\rho$ and the limiting values of $\wb{\cost}_X(\epsilon)$.

\paragraph{Finalizing Theorem~\ref{thm:cost-interpolator-isotropic}
  \ref{item:relative-costs-ols}: The cost of minimum norm interpolation.}

As in our derivation of Eq.~\eqref{eq:using-stieljes-transform}, we can
apply \citet[Lemma~3.11]{BaiSi10} to the integral form of
Lemma~\ref{lemma:cost-difference-ols}. Recalling \citeauthor{BaiSi10}'s result,
we have
\begin{align}
  \label{eqn:bai-silverstein-transform}
  \int \frac{1}{s + \sigma^2} dH(s)
  = \frac{1 - 1/\gamma + \sigma^2 - \sqrt{\prn{1 - 1 / \gamma + \sigma^2 }^2
      + 4 \sigma^2/\gamma}}{- 2 \sigma^2/\gamma}.
\end{align}
As
\begin{equation*}
  \left[\prn{1 - \frac{1}{\gamma} + \sigma^2} -
    \sqrt{\prn{1 - \frac{1}{\gamma} + \sigma^2}^2
      + \frac{4 \sigma^2}{\gamma}}\right]
  \left[\prn{1 - \frac{1}{\gamma} + \sigma^2} +
    \sqrt{\prn{1 - \frac{1}{\gamma} + \sigma^2}^2
      + \frac{4 \sigma^2}{\gamma}}\right]
  = \frac{4 \sigma^2}{\gamma},
\end{equation*}
we then use that $H$ has support bounded away from zero to immediately
obtain
\begin{align*}
  \int \frac{1}{s} dH(s)
  & = \lim_{\sigma \downarrow 0} \frac{1 - 1/\gamma + \sigma^2 - \sqrt{\prn{1 - 1 / \gamma + \sigma^2 }^2 + 4
			\sigma^2/\gamma}}{- 2 \sigma^2/\gamma}  \nonumber \\
	& = \lim_{\sigma \downarrow 0} \frac{2}{  \sqrt{\prn{1 - 1 / \gamma + \sigma^2 }^2 + 4 \sigma^2/\gamma} + \prn{1 - 1 / \gamma + \sigma^2 }}
  = \frac{1}{1 - 1/\gamma}.
\end{align*}
As $\frac{\sigma^2}{s(s + \sigma^2)} = \frac{1}{s} - \frac{1}{s +
  \sigma^2}$, we then again use identity~\eqref{eqn:bai-silverstein-transform}
and Lemma~\ref{lemma:cost-difference-ols} to see that
\begin{align*}
  \lefteqn{
    \frac{\sigma^2}{\gamma} \cdot \prn{	\int \frac{1}{s} dH(s) - \int \frac{1}{s + \sigma^2} dH(s) }} \\
  & = \frac{\sigma^2}{\gamma} \cdot \prn{\frac{1}{1 - 1/\gamma}  - \frac{2}{  \sqrt{\prn{1 - 1 / \gamma + \sigma^2 }^2 + 4 \sigma^2/\gamma} + \prn{1 - 1 / \gamma + \sigma^2 }} } \nonumber \\
	& = \frac{\sigma^2}{\gamma} \cdot \prn{\frac{1}{1 - 1/\gamma}  - \frac{2}{ \prn{1 - 1 / \gamma + \sigma^2 } + \frac{4 \sigma^2/\gamma}{2 \prn{1 - 1 / \gamma + \sigma^2 }} + \prn{1 - 1 / \gamma + \sigma^2 } +o(\sigma^2)} } \nonumber \\
	& = \frac{\sigma^2}{\gamma} \cdot \prn{\frac{1}{1 - 1/\gamma}  - \frac{1}{ 1 - 1 / \gamma + \frac{\sigma^2 }{1-1/\gamma}+o(\sigma^2)} }  \nonumber \\
	& = \frac{\sigma^4}{\gamma\prn{1 - 1/\gamma}^3} + o(\sigma^4) \, ,
\end{align*}
where we use the Taylor expansions $\sqrt{x^2 + t} = x + \frac{t}{2x} +
o(t^2)$ and $\frac{1}{x + t} = \frac{1}{x} - \frac{t}{x^2} + o(t^2)$, valid
for any fixed $x > 0$.

\paragraph{Proving Theorem~\ref{thm:cost-interpolator-isotropic}
  \ref{item:interp-threshold-ols}: interpolation threshold.}
To obtain the threshold value $\rhools$, we derive the limit $\lim_{n \to
  \infty} \wb{\cost}_X(\epsilon)$ for any $\epsilon >0$. As
Lemma~\ref{lemma:cost-difference-ols} shows,
\begin{align*}
  \lim_{n \to \infty} \prn{\cost_X(\epsilon) - \wb{\cost}_X(\epsilon)}
  = \frac{\sigma^4}{\gamma}  \int \frac{1}{s(s+\sigma^2)} dH(s) .
\end{align*}
Applying Theorem~\ref{thm:cost-isotropic} for the limiting value of
$\cost_X(\epsilon)$, we recall the definition~\eqref{eq:threshold-isotropic}
of $\epsilon_\sigma^2 = \int \frac{\sigma^4}{s + \sigma^2} dH(s)$. Choose
$\rho = \rho(\epsilon)$ to be $\rho(\epsilon) = 0$ if $\epsilon <
\epsilon_\sigma$ and to satisfy $\epsilon^2 =\int \frac{\sigma^4}{(1 - \rho
  s)^2 \prn{s + \sigma^2}} dH(s)$ when $\epsilon \geq \epsilon_\sigma$, as
in Eq.~\eqref{eq:rho-epsilon-equation} in Theorem~\ref{thm:cost-isotropic},
which decreases continuously to $\rho(\epsilon_\sigma) = 0$. The theorem then
implies
\begin{align*}
  \lim_{n \to \infty} \cost_X(\epsilon) = \frac{\rho^2}{\gamma} \int \frac{\sigma^4 s}{(1-\rho s)^2 (s + \sigma^2)} dH(s).
\end{align*}
Adding and subtracting $\cost_X(\epsilon)$, we therefore
have with probability 1  that
\begin{align}
  \nonumber
  \lim_{n \to \infty} \wb{\cost}_X(\epsilon) & = 	\lim_{n \to \infty} \cost_X(\epsilon)
  - \lim_{n \to \infty} \prn{\cost_X(\epsilon) - \wb{\cost}_X(\epsilon)} \\
  & = \frac{\rho^2}{\gamma}
  \int \frac{\sigma^4 s}{(1-\rho s)^2 (s + \sigma^2)} dH(s)
  - \frac{\sigma^4}{\gamma}  \int \frac{1}{s(s+\sigma^2)} dH(s)
  \label{eqn:rho-ols-cost-limit}
\end{align}
(compare with Eq.~\eqref{eq:rho-sigma-equation}).
Notably, $\rho = \rho(\epsilon)$ satisfies $\rho = 0$
whenever $\epsilon < \epsilon_\sigma$, so
that
\begin{equation*}
  \lim_{n \to \infty} \wb{\cost}_X(\epsilon) = -\frac{\sigma^4}{\gamma} \int
  \frac{1}{s(s+\sigma^2)} dH(s)< 0
\end{equation*}
for $\epsilon < \epsilon_\sigma$.

Now, consider the $\rhools$ solving identity~\eqref{eq:rho-sigma-equation}
and the associated value $\epsols$, where it is evident
that $\rhools > 0$. Then the preceding
calculations yield immediately that
\begin{align*}
  \lim_{n \to \infty} \wb{\cost}_X(\epsols)
  = \frac{\sigma^4}{\gamma} \cdot \prn{
    \rhools^2 \int \frac{s}{(1 - \rhools s)^2
      \prn{s + \sigma^2}} dH(s) -\int \frac{1}{s(s+\sigma^2)} dH(s) } = 0.
\end{align*}
Because the value $\rho = \rho(\epsilon)$ solving the
identity~\eqref{eq:rho-epsilon-equation} is increasing in $\epsilon \ge
\epsilon_\sigma$, we conclude that $\rho(\epsilon) > \rhools$ for $\epsilon
> \epsols$ and $\epsols > \epsilon_\sigma$.  Combining everything to this
point and the limit~\eqref{eqn:rho-ols-cost-limit}, we see that
\begin{equation*}
  \lim_{n \to \infty} \wb{\cost}_X(\epsilon)
  \begin{cases} > 0 & \mbox{if~} \epsilon > \epsols \\
    < 0 & \mbox{if~} \epsilon < \epsols.
  \end{cases}
\end{equation*}
 
Lastly, we provide the concrete claimed bounds on $\epsols$ in terms of
$\epsilon_\sigma$. We have already seen that $\epsols > \epsilon_\sigma$,
and so the claimed upper bound revolves around lower bounding $\rhools$ so
that we may provide an upper bound on $\epsols = \int \frac{\sigma^4}{(1 -
  \rhools s)^2 (s + \sigma^2)} dH(s)$. To that end, note that
identity~\eqref{eq:rho-sigma-equation} gives a lower bound for
$\rhools$: as
\begin{align*}
  \int \prn{\frac{\rhools^2 s^2}{(1 - \rhools s)^2} - 1} \cdot \frac{1}{s\prn{s + \sigma^2}} dH(s) = 0 \, ,
\end{align*}
we must have
\begin{equation*}
  \sup_{s \in [\lambda_-, \lambda_+]} \frac{\rhools^2 s^2}{(1 - \rhools s)^2} - 1
  \geq 0,
  ~~ \mbox{so} ~~
  \rhools \ge \frac{1}{2 \lambda_+}.
\end{equation*}
Invoking the lower bound
$\rhools \cdot 2 \lambda_+ \geq 1$ and that $s / \lambda_- \geq 1$ on the
support of $H$, we have
\begin{align*}
  \epsols^2
  & = \int \frac{\sigma^4}{(1 - \rhools s)^2 \prn{s + \sigma^2}} dH(s) \nonumber \\
  & \le \frac{4 \lambda_+^2}{\lambda_-} \cdot \sigma^4 \cdot \rhools^2 \int \frac{s}{(1 - \rhools s)^2 \prn{s + \sigma^2}} dH(s)
  = \frac{4 \lambda_+^2 \sigma^4}{\lambda_-} \int \frac{1}{s(s+\sigma^2)} dH(s),
\end{align*}
where we used the identity~\eqref{eq:rho-sigma-equation}.
Noting that $\frac{1}{s} \le \frac{1}{\lambda_-}$ and using the
definition~\eqref{eq:threshold-isotropic} of
$\epsilon_\sigma = \int \frac{\sigma^4}{s + \sigma^2} dH(s)$ gives the
final bound
that $\epsols^2 \le \frac{4 \lambda_+^2}{\lambda_-^2} \epsilon_\sigma^2$,
as desired.

%% file: sections/proof-anisotropic.tex

\section{Proof of Theorem~\ref{thm:cost-general-cov}}
  
\label{proof:cost-general-cov}

The proof follows a similar approach to that we use in the proof of
Theorem~\ref{thm:cost-isotropic} in Section~\ref{proof:cost-isotropic}: we
compute formulae for the training and prediction errors conditional on the
data matrices $X$, then use these to provide the bounds on the memorization
threshold and costs for fitting to accuracy worse than that threshold. While
in the proof of Theorem~\ref{thm:cost-isotropic}, we could develop explicit
spectral limits for the error measures of interest, here exact forms are
difficult, but we can obtain tight enough bounds (mitigated by the condition
number $\kappa$ of the covariance $\Sigma$ of the data vectors $x$) to
give the desired results.
With that in mind, we note that Lemmas~\ref{lem:closed-form-errors},
\ref{lem:strong-duality}, and \ref{lem:rho-error-forms} all continue to
hold, so that the reduction via strong duality applies. In particular, the
optimal linear estimator $A$ in the form $\est{\theta} = A y$ continues to
take the form $A(\rho, \Sigma)$ in~\eqref{eqn:A}.

Throughout the proof, we let $\lambda_1 \ge \lambda_2 \ge \cdots \ge
\lambda_n \ge 0$ denote the singular values of $X$ and $\mu_1 \geq \mu_2
\geq \cdots \geq \mu_n \ge 0$ those of $Z$, and so the
empirical spectral c.d.f.s of $\frac{1}{d} XX^\top$ and $\frac{1}{d}
ZZ^\top$ are (respectively)
\begin{equation*}
  G_n(s) \defeq \frac{1}{n} \sum_{i=1}^n \ind_{\lambda_i^2/d \leq s}
  ~~ \mbox{and} ~~
  H_n(s) \defeq \frac{1}{n} \sum_{i=1}^n \ind_{\mu_i^2/d \leq s}.
\end{equation*} 
By the Marchenko-Pastur and deformed Marchenko-Pastur laws
(Lemmas~\ref{thm:MP-law} and~\ref{thm:MP-deformed-law}),
$G_n$ and $H_n$ converge weakly (almost surely) to c.d.f.s
$G$ and $H$, respectively. Again, we only need to prove under Assumption~\ref{assp:linear} by applying Theorem~\ref{thm:strong-duality-hilbert}.

\paragraph{Part I: Memorization threshold.}
We begin with the expansion of $\epsdeformed$ and the bound $\epsdeformed^2
\le \epsilon_{\sqrt{\kappa} \sigma}^2 / \kappa$.
Rewriting $\epsilon_\sigma$ and $\epsdeformed$ in terms of the limits
arising from their respective Marchenko-Pastur laws,
we have
\begin{align*}
  \epsdeformed^2 & = \int \frac{\sigma^4}{s +\sigma^2} dG(s) = \lim_{n \to \infty} \int \frac{\sigma^4}{s +\sigma^2} dG_n(s) = \lim_{n \to \infty} \frac{d\sigma^4}{n} \Tr \prn{(XX^\top + d \sigma^2 I)^{-1}} \, , \nonumber \\
  \epsilon_{\sqrt{\kappa}\sigma}^2 / \kappa & =  \int \frac{\kappa\sigma^4}{s + \kappa\sigma^2} dH(s) = \lim_{n \to \infty} \int \frac{\kappa\sigma^4}{s + \kappa\sigma^2} dH_n(s) = \lim_{n \to \infty} \frac{d\sigma^4}{n} \Tr \prn{(ZZ^\top /\kappa + d \sigma^2 I)^{-1}} \, .
\end{align*}
As $XX^\top = Z \Sigma Z^\top \succeq ZZ^\top / \kappa$, we have $\Tr
(ZZ^\top /\kappa + d \sigma^2 I)^{-1} \geq \Tr (XX^\top + d \sigma^2
I)^{-1}$ and thus $\epsdeformed^2 \leq \epsilon_{\sqrt{\kappa} \sigma}^2 /
\kappa$.

\paragraph{Part II: No cost below threshold.}
It is immediate via Lemma~\ref{lem:strong-duality} that the
global minimizer for the unconstrained problem~\eqref{eq:not-fitting-problem}
(with $\epsilon = 0$) is $A(0, \Sigma)$, that is,
$\rho = 0$ as the
constraint is inactive and
\begin{align*}
  A(0, \Sigma)  = X^\top (XX^\top + d \sigma^2 I)^{-1}
  = (X^\top X + d \sigma^2 I)^{-1} X^\top.
\end{align*}
Then as usual $\inf_{\est{\theta} \in \mc{H}(0)} \pred_X(\est{\theta}) =
\pred_X(\est{\theta}_{d \sigma^2})$, where we recall $\est{\theta}_{d
  \sigma^2}$ is the ridge estimator.  To prove that $\lim_{n \to \infty}
\cost_X(\epsilon) = 0$ when $\epsilon < \epsdeformed$, it is
thus sufficient to show that $\est{\theta}_{d \sigma^2}$ is
contained in $\mathcal{H}(\epsilon)$ eventually, which amounts to proving
\begin{align*}
  \liminf_{n \to \infty}	\train_X\prn{\est{\theta}_{d \sigma^2}} =
  \liminf_{n \to \infty} \mathcal{T}(A(0, \Sigma); \Sigma)> \epsilon^2.
\end{align*}
Invoking the expansion of $\mc{T}(A(\rho, \Sigma); \Sigma)$
in Lemma~\ref{lem:rho-error-forms} and setting $\rho = 0$, we obtain
\begin{align*}
  \mc{T}(A(0, \Sigma); \Sigma)
  & =  \frac{d \sigma^4}{n}
  \Tr \prn{X^\top X \prn{X^\top X}^\dagger \prn{X^\top X + d\sigma^2 I}^{-1} }
  = \int \frac{\sigma^4 }{s + \sigma^2} dG_n(s) .
\end{align*}
By weak convergence,
\begin{equation*}
  \lim_{n \to \infty}
  \mathcal{T}(A(0, \Sigma); \Sigma)
  = \lim_{n \to \infty} \int \frac{\sigma^4 }{s + \sigma^2} dG_n(s)
  = \int \frac{\sigma^4 }{s + \sigma^2} dG(s)
  = \epsdeformed^2 > \epsilon^2,
\end{equation*}
so indeed we have $\est{\theta}_{d\sigma^2} \in \mc{H}(\epsilon)$ as desired.

\paragraph{Part III: Cost of not-fitting above threshold.}

Our starting point is to demonstrate the existence and uniqueness
of $\rhodeformed \in
\openright{0}{\lambda_+^{-1}}$ solving the
identity~\eqref{eq:rho-epsilon-equation-general-cov}. For this, we
note that the difference
\begin{equation*}
  \Delta_H(\rho) \defeq \int\left[\frac{1}{(1 - \rho s)^2
      (s + \kappa \sigma^2)} - \frac{1}{s + \sigma^2} \right] dH(s)
\end{equation*}
is monotone increasing in $\rho$, and $\Delta_H(0) = 0$. That
$\Delta_H(\rho) \to \infty$ as $\rho \uparrow \lambda_+^{-1}$ is then an
immediate consequence of the expansion~\eqref{eq:integral-blows-up} of the
left integrand above.

We turn to the second claim in part~\ref{item:cost-not-fit-general}: the
lower bound on $\cost_X(\epsilon)$. We (roughly) reduce the general
covariance case to the isotropic case, then apply our previous results and
techniques. To do so, we require the following lemma, which upper-bounds the
training error growth and lower-bounds the prediction error growth. The
proof is essentially tedious algebraic manipulations, so we defer
it to Appendix~\ref{proof:growth-control-via-isotropic}.
\begin{lemma} \label{lem:growth-control-via-isotropic}
  Let the same conditions of Lemma~\ref{lem:rho-error-forms} hold and
  assume $\rho \lambda_1^2 /d < 1$.
  Then
  \begin{align*}
    \mc{P}(A(\rho,\Sigma); \Sigma)
    - \mc{P}(A(0,\Sigma); \Sigma)
    & \geq \frac{\rho^2  \sigma^4}{d} \Tr \prn{\prn{I- \frac \rho d  ZZ^\top}^{-2}  \frac{ZZ^\top}{d} \cdot  \prn{ \frac{Z Z^\top}{d} +\sigma^2 I}^{-1}  }, \\
    \mc{T}(A(\rho,\Sigma); \Sigma)
    - \mc{T}(A(0,\Sigma); \Sigma) 
    & \le \frac{\kappa \sigma^4}{n}
    \Tr\left[ \prn{\prn{I - \frac{\rho}{d} Z Z^\top}^{-2} - I}
      \prn{\frac{1}{d} ZZ^\top + \kappa \sigma^2 I}^{-1}\right].
  \end{align*}
\end{lemma}

We use the upper and lower bounds in
Lemma~\ref{lem:growth-control-via-isotropic}, coupled with the strong
duality guarantees in Lemma~\ref{lem:strong-duality} (and the
identities~\eqref{eqn:A}), to prove the desired growth of the
$\cost_X(\epsilon)$. Consider any $0 \le \rho < \rhodeformed$, where
$\rhodeformed$ satisfies the
identity~\eqref{eq:rho-epsilon-equation-general-cov}.
By construction and duality,
$A(\rho, \Sigma)$ is the optimal solution to the problem
\begin{equation*}
  \begin{aligned}
    \minimize_{A \in \real^{d \times n}} & ~~ \mathcal{P}(A; \Sigma)  \\
    \subjectto & ~~ \mc{T}(A(\rho, \Sigma); \Sigma)
    - \mathcal{T}(A; \Sigma) \le 0.
  \end{aligned}
\end{equation*}
Thus, whenever $\mathcal{T}(A(\rho, \Sigma); \Sigma) < \epsilon^2$ it
holds that
\begin{align}
  \cost_X (\epsilon) & \geq \mathcal{P}(A(\rho, \Sigma); \Sigma)
  - \mathcal{P}(A(0, \Sigma); \Sigma).
  \label{eqn:intermediate-cost-bound}
\end{align}
Therefore, to prove that $\cost_X(\epsilon)$ grows it is sufficient
to show that  eventually $\mc{T}(A(\rho, \Sigma); \Sigma) < \epsilon$ for
our chosen $\rho$ and
provide lower bounds on the difference
$\mathcal{P}(A(\rho, \Sigma); \Sigma)
- \mathcal{P}(A(0, \Sigma); \Sigma)$.

To that end, let us take limits of $\mc{T}$. Applying
the upper bound in Lemma~\ref{lem:growth-control-via-isotropic}, we have
\begin{align*}
  \lefteqn{
    \limsup_{n \to \infty} \mathcal{T}(A(\rho, \Sigma); \Sigma)} \\
  & \leq \limsup_{n \to \infty} \mathcal{T}(A(0, \Sigma); \Sigma)
  + \limsup_{n \to \infty} \frac{\kappa \sigma^4 }{n} \Tr
  \left[\prn{ \prn{I - \frac{\rho}{d} ZZ^\top}^{-2} - I}
    \prn{\frac{1}{d} ZZ^\top + \kappa \sigma^2 I}^{-1} \right]
  \nonumber \\
  & = \epsdeformed^2
  +  \limsup_{n \to \infty}
  \kappa \sigma^4 \int \frac{\rho s(2 - \rho s)}{(1 - \rho s)^2 \prn{s + \kappa \sigma^2}} d H_n(s)
\end{align*}
with probability 1. As $\rho < \rhodeformed < \lambda_+^{-1}$, the quantity
$\frac{s(2 - \rho s)}{(1 - \rho s)^2 \prn{s + \kappa \sigma^2}}$ is
eventually bounded on the support $[\lambda_-, \lambda_+] + o(1)$ of $H_n$
by the Bai-Yin law (Lemma~\ref{thm:BY-law}), and so with probability one
\begin{align*}
  \kappa \sigma^4 \int \frac{\rho s(2 - \rho s)}{(1 - \rho s)^2 \prn{s + \kappa \sigma^2}} d H_n(s)
  & \to \kappa \sigma^4 \int \frac{\rho s(2 - \rho s)}{(1 - \rho s)^2 \prn{s + \kappa \sigma^2}} d H(s) \\
  & =  \kappa \sigma^4 \int \prn{\frac{1}{(1 - \rho s)^2 \prn{s + \kappa \sigma^2}} - \frac{1}{s + \kappa \sigma^2}} d H(s) \nonumber \\
  & < \kappa \sigma^4 \int \prn{\frac{1}{(1 - \rhodeformed s)^2
      \prn{s + \kappa \sigma^2}} - \frac{1}{s + \kappa \sigma^2}} d H(s) \nonumber \\
  & = \epsilon^2 - \epsdeformed^2,
\end{align*}
where the last line follows from the
definition~\eqref{eq:rho-epsilon-equation-general-cov} of $\rhodeformed$.
In particular, with probability 1 we have
\begin{align*}
  \limsup_{n \to \infty} \mathcal{T}(A(\rho, \Sigma); \Sigma) & <
  \epsdeformed^2 + \epsilon^2 - \epsdeformed^2 = \epsilon^2,
\end{align*}
and therefore inequality~\eqref{eqn:intermediate-cost-bound}
implies that with probability 1,
\begin{equation*}
  \liminf_{n \to \infty} \cost_X(\epsilon)
  \ge \liminf_{n \to \infty}
  \left[\mc{P}(A(\rho, \Sigma); \Sigma) - \mc{P}(A(0, \Sigma); \Sigma)\right].
\end{equation*}

We now apply Lemma~\ref{lem:growth-control-via-isotropic} again,
invoking the lower bound on the prediction errors to obtain
\begin{align*}
	\liminf_{n \to \infty} \cost_X (\epsilon) & \geq \lim_{n \to \infty} \frac{\rho^2  \sigma^4}{d} \Tr \prn{\prn{I- \frac{\rho}{d}  ZZ^\top}^{-2}  \frac{ZZ^\top}{d} \cdot  \prn{ \frac{Z Z^\top}{d} +\sigma^2 I}^{-1}  } \nonumber \\
	&= \lim_{n \to \infty} \rho^2  \sigma^4 \cdot \frac{n}{d} \cdot \int \frac{s}{\prn{1 - \rho s}^2 (s + \sigma^2)} dH_n(s) \nonumber \\
	& = \frac{\rho^2}{\gamma} \int \frac{\sigma^4 s}{(1-\rho s)^2 (s + \sigma^2)} dH(s) \, .
\end{align*}
Taking $\rho \uparrow \rhodeformed$ yields the second
claim of part~\ref{item:cost-not-fit-general}.

Our last step is to prove a concrete lower bound showing that
$\cost_X(\epsilon)$ grows linearly in $\epsilon^2$ provided that $\epsilon^2
\ge \frac{2 \kappa \sigma^4}{\lambda_- + \kappa \sigma^2}$, in parallel to
the result in Lemma~\ref{lem:cost-isotropic-mid-3}. We state
a small integral inequality:
\begin{lemma}
  \label{lemma:rho-to-epsilon-integral}
  Let $\rho = \rhodeformed$ solve the fixed
  point~\eqref{eq:rho-epsilon-equation-general-cov}. Then
  \begin{equation*}
    \int \frac{\kappa \sigma^4}{(1 - \rho s)^2(s + \kappa \sigma^2)} dH(s)
    \ge \epsilon^2.
  \end{equation*}
\end{lemma}
\begin{proof}
  The identity~\eqref{eq:rho-epsilon-equation-general-cov} shows
  that
  the integral in the statement of the lemma equals
  $\int \frac{\kappa \sigma^4}{s + \kappa \sigma^2} dH(s)
  + \epsilon^2 - \epsdeformed^2$. Recall that
  by
  part~\ref{item:general-threshold} of Theorem~\ref{thm:cost-general-cov},
  we have
  $\epsdeformed^2 \le \epsilon^2_{\sqrt{\kappa} \sigma} / \kappa
  = \int \frac{\kappa \sigma^4}{s + \kappa \sigma^2} dH(s)$.
\end{proof}

Taking $\rho =
\rhodeformed$ to solve the fixed
point~\eqref{eq:rho-epsilon-equation-general-cov}, we apply
the second claim in part~\ref{item:cost-not-fit-general} to see that
\begin{align}
  \nonumber \liminf_{n \to \infty}
  \cost_X(\epsilon)
  & \ge \frac{\rho^2}{\gamma}
  \int \frac{\sigma^4 s}{(1 - \rho s)^2 (s + \sigma^2)} dH(s)
  \ge \frac{\rho^2 \lambda_-}{\gamma} \int \frac{\sigma^4}{
    (1 - \rho s)^2 (s + \sigma^2)} dH(s) \\
  & \ge \frac{\rho^2 \lambda_-}{\kappa \gamma}
  \int \frac{\kappa \sigma^4}{(1 - \rho s)^2 (s + \kappa \sigma^2)} dH(s)
  \ge \frac{\rho^2 \lambda_-}{\kappa \gamma} \epsilon^2
  \label{eqn:cost-with-kappa}
\end{align}
by Lemma~\ref{lemma:rho-to-epsilon-integral}.
It remains to lower bound $\rho = \rhodeformed < \lambda_+^{-1}$. For this,
we observe that
\begin{align*}
  \frac{1}{(1 - \rho \lambda_+)^2}
  \ge \int \frac{1}{(1 - \rho s)^2} dH(s)
  & \ge \frac{\lambda_- + \kappa \sigma^2}{\kappa \sigma^4}
  \int \frac{\kappa \sigma^4}{(1 - \rho s)^2 (s + \kappa \sigma^2)}
  dH(s)
  \ge \frac{\lambda_- + \kappa \sigma^2}{\kappa \sigma^4}
  \cdot \epsilon^2,
\end{align*}
again applying Lemma~\ref{lemma:rho-to-epsilon-integral}.
In particular, whenever $\frac{\lambda_+ + \kappa \sigma^2}{\kappa \sigma^4}
\epsilon^2 \ge 2$, we obtain
$(1 - \rho \lambda_+)^{-2} \ge 2$, or
$\rhodeformed \ge \frac{1}{\lambda_+}(1 - 1/\sqrt{2})$.
Substituting in inequality~\eqref{eqn:cost-with-kappa}
gives the lower bound on $\liminf_n \cost_X(\epsilon)$.

\subsection{Proof of Lemma~\ref{lem:growth-control-via-isotropic}}
\label{proof:growth-control-via-isotropic}

We prove each claim of the lemma in turn. For the first,
we use the shorthand $\Delta_\mc{P}(\rho) \defeq \mc{P}(A(\rho, \Sigma); \Sigma)
- \mc{P}(A(0, \Sigma); \Sigma)$. Then applying
Lemma~\ref{lem:rho-error-forms}, we have
\begin{equation*}
  \Delta_\mc{P}(\rho)
  = \frac{\rho^2  \sigma^4}{d} \Tr
  \prn{\prn{\Sigma - \frac \rho d X^\top X}^{-1}\Sigma
    \prn{\Sigma - \frac \rho d X^\top X}^{-1} X^\top
    \prn{X X^\top + d \sigma^2 I}^{-1}  X},
\end{equation*}
and making the substitution
$X = Z\Sigma^{\half}$ immediately yields
\begin{align*}
  \lefteqn{\Delta_{\mc{P}}(\rho)} \\
  & = \frac{\rho^2  \sigma^4}{d} \Tr \prn{\prn{
      \Sigma - \frac \rho d \Sigma^{\half} Z^\top Z \Sigma^{\half}}^{-1}
    \Sigma \prn{\Sigma - \frac \rho d\Sigma^{\half} Z^\top
      Z \Sigma^{\half}}^{-1}\Sigma^{\half} Z^\top
    \prn{ Z \Sigma Z^\top + d \sigma^2 I}^{-1} Z \Sigma^{\half}} \\
  & = \frac{\rho^2  \sigma^4}{d} \Tr \prn{Z \prn{I- \frac \rho d  Z^\top Z}^{-2}  Z^\top \cdot  \prn{ Z \Sigma Z^\top + d \sigma^2 I}^{-1}  }.
\end{align*}
As $Z (I- \frac \rho d Z^\top Z)^{-2} Z^\top \succeq 0$ and $Z \Sigma Z^\top
+ d\sigma^2 I \preceq ZZ^\top + d\sigma^2 I$ as $\Sigma \preceq I$ by
assumption, we can leverage that the mapping $A \mapsto \Tr(AC)$ is
increasing in the positive definite order for $C \succeq 0$ to obtain that
\begin{align*}
  \Delta_{\mc{P}}(\rho)
  & \ge \frac{\rho^2  \sigma^4}{d} \Tr \prn{Z \prn{I- \frac \rho d  Z^\top Z}^{-2}
    Z^\top \cdot  \prn{ ZZ^\top + d \sigma^2 I}^{-1}  } \\
  & = \frac{\rho^2  \sigma^4}{d} \Tr \prn{\prn{I- \frac \rho d  ZZ^\top}^{-2}  \frac{ZZ^\top}{d} \cdot  \prn{ \frac{Z Z^\top}{d} +\sigma^2 I}^{-1}  },
\end{align*}
where in the last line we used the identity $Z (I - \frac{\rho}{d}Z^\top
Z)^{-1} = (I - \frac{\rho}{d}ZZ^\top)^{-1} Z$. This gives the first claim of
Lemma~\ref{lem:growth-control-via-isotropic}.

We turn to the upper bound on the training error, for which we use the
shorthand $\Delta_{\mc{T}}(\rho) \defeq \mc{T}(A(\rho, \Sigma); \Sigma) -
\mc{T}(A(0, \Sigma); \Sigma)$. Beginning from
the expansion of $\mc{T}$ in Lemma~\ref{lem:rho-error-forms}, we have
\begin{align}
  \nonumber
  \frac{n}{d \sigma^4}
  \Delta_{\mc{T}}(\rho)
  & = \Tr \left[
    \Sigma \prn{\Sigma - \frac \rho d X^\top X}^{-1} X^\top X \prn{\Sigma -
      \frac{\rho}{d} X^\top X}^{-1}
    \Sigma \prn{X^\top X}^\dagger \prn{X^\top X + d\sigma^2 I}^{-1}   \right] \\
  & \qquad - \Tr\left[X^\top X (X^\top X)^\dagger (X^\top X + d \sigma^2 I)^{-1}
    \right].
  \label{eqn:train-error-aniso-diff}
\end{align}
Leveraging the identities $X = Z \Sigma^\half$ and that
\begin{align*}
  (X^\top X)^\dagger (X^\top X + d\sigma^2 I)^{-1}
  & = X^\top (XX^\top)^{-2} (XX^\top + d\sigma^2 I)^{-1} X \\
  & = \Sigma^\half Z^\top (Z \Sigma Z^\top)^{-2} (Z\Sigma Z^\top + d \sigma^2 I)^{-1}
  Z\Sigma^\half,
\end{align*}
the right hand side of the expansion~\eqref{eqn:train-error-aniso-diff}
becomes
\begin{align*}
  \lefteqn{\Tr\left[
      \prn{\Sigma^\half \prn{I - \frac{\rho}{d} Z^\top Z}^{-1} Z^\top Z
        \prn{I - \frac{\rho}{d} Z^\top Z}^{-1}\Sigma^\half - \Sigma^\half Z^\top Z
        \Sigma^\half}
      X^\top (X X^\top)^{-2} (XX^\top + d \sigma^2 I)^{-1} X\right]} \\
  & = \Tr\left[
    \Sigma^\half \prn{\prn{I - \frac{\rho}{d} Z^\top Z}^{-1}
      Z^\top Z \prn{I - \frac{\rho}{d} Z^\top Z}^{-1} - Z^\top Z}
    \Sigma Z^\top (Z\Sigma Z^\top)^{-2}
    (Z \Sigma Z^\top + d \sigma^2 I)^{-1} Z \Sigma^\half \right] \\
  & = \Tr\left[
    \Sigma^\half \prn{\prn{I - \frac{\rho}{d} Z^\top Z}^{-2} - I} Z^\top
    (Z \Sigma Z^\top)^{-1} \prn{Z \Sigma Z^\top + d \sigma^2 I}^{-1} Z \Sigma^\half
    \right],
\end{align*}
where we have used that $(I - \frac{\rho}{d} Z^\top Z)^{-1}$ and
$Z^\top Z$ commute and eliminated one inverse of $Z\Sigma Z^\top$.
The singular value decomposition gives the equality
$(I - \frac{\rho}{d} Z^\top Z)^{-2} Z^\top
= Z^\top (I - \frac{\rho}{d} ZZ^\top)^{-2}$, where $I$ is an identity
matrix of appropriate size. The cyclic property of the trace and
that $(Z\Sigma Z^\top)^{-1}$ and
$(Z\Sigma Z^\top + d\sigma^2 I)^{-1}$ comute then allows us to substitute into
the identity~\eqref{eqn:train-error-aniso-diff} to obtain
\begin{align*}
  \Delta_{\mc{T}}(\rho)
  & =   \frac{d \sigma^4}{n}
  \Tr\left[\prn{\prn{I - \frac{\rho}{d} Z Z^\top}^{-2} - I}
    \prn{Z\Sigma Z^\top + d \sigma^2 I}^{-1}\right].
\end{align*}
Lastly, we again use the monotonicity of
$A \mapsto \Tr(A C)$ for $C \succeq 0$ and that
$Z \Sigma Z^\top + d \sigma^2 I \succeq ZZ^\top / \kappa + d \sigma^2 I$ to
get claimed upper bound in the lemma.

%% file: sections/proof-hilbert-duality.tex
\section{Proof of Theorem~\ref{thm:strong-duality-hilbert}}  \label{proof:strong-duality-hilbert}

We provide the proof conditional on $X$, implicitly conditioning
throughout. As $\Hlin(\epsilon) \subset \Hsq(\epsilon)$, we
only need to show
\begin{align*}
  \inf_{\est{\theta} \in \Hsq(\epsilon)} \pred_X \prn{\est{\theta}} \geq \min_{\est{\theta} \in \Hlin(\epsilon)}  \pred_X \prn{\est{\theta}} \, .
\end{align*}
First we note that in the Gaussian setting
that $\theta \sim \normal(0, \frac{1}{d} I)$,
we have $y = X\theta + \varepsilon \sim \normal(0,
\frac{XX^\top}{d} + \sigma^2 I)$. By a standard
calculation,
the conditional distribution of $\theta$ given $y$ is
\begin{align*}
  \theta \mid y \sim
  \normal \prn{\prn{X^\top X + d \sigma^2 I}^{-1}X^\top y,
    \sigma^2 \prn{X^\top X +d \sigma^2 I}^{-1}},
\end{align*}
and therefore for any $\est{\theta}(X,y) \in \Hsq$,
\begin{align*}
  \pred_X \prn{\est{\theta}}
  & = \Ep_y \brk{\Ep_{\theta \mid y} \brk{\norm{\Sigma^{\frac{1}{2}} \prn{\est{\theta} - \theta}}_2^2 \mid y}} \\ 
  & = \Ep_y \brk{\norm{\Sigma^{\frac{1}{2}} \prn{\est{\theta} - \prn{X^\top X + d \sigma^2 I}^{-1}X^\top y}}_2^2 + \sigma^2 \Tr \prn{\Sigma  \prn{X^\top X +d \sigma^2 I}^{-1}} }. \nonumber
\end{align*}
Notably, the posterior mean $\E[\theta \mid y]$ always minimizes the prediction
risk.
By Lemma~\ref{lem:strong-duality} we know there is a $\rho$ such that
$\est{\theta}(\rho) := A(\rho, \Sigma) y$ is optimal for
problem~\eqref{eq:not-fitting-problem} where
\begin{align*}
  A(\rho, \Sigma) = \prn{I - \rho \sigma^2 \prn{\Sigma - \frac{\rho}{d} X^\top X}^{-1}}  (X^\top X + d \sigma^2 I)^{-1} X^\top.
\end{align*}
We consider two cases, depending on whether the value of the dual variable
$\rho = 0$ or $\rho > 0$.

\paragraph{Case I: $\rho = 0$.}
In this case $\est{\theta}(0) = \prn{X^\top X + d \sigma^2 I}^{-1}X^\top y
\in \Hlin(\epsilon) \subset \Hsq(\epsilon)$. But this
is the posterior mean, that is, $\est{\theta}(0) = \E[\theta \mid y]$, which
is thus optimal.


\paragraph{Case II: $\rho > 0$.}
As $\pred_X (\est{\theta}(\rho))$ is continuous in $\rho$,
if we can prove for any $\est{\theta} \in \Hsq(\epsilon)$ and any
$0 \leq \wb{\rho} < \rho$ that
\begin{align}
  \pred_X \prn{\est{\theta}} \geq \pred_X \prn{\est{\theta}(\wb{\rho})},
  \label{eq:strong-duality-hilbert-claim}
\end{align}
taking $\wb{\rho} \uparrow \rho$ completes the proof. (Note that $\rho$ is
the optimal dual variable for problem~\eqref{eq:not-fitting-problem},
and so $\est{\theta}(\wb{\rho}) \in \mc{F}(\epsilon)$.)

To show claim~\eqref{eq:strong-duality-hilbert-claim}, let $\mu = \normal(0,
\frac{1}{d} XX^\top + \sigma^2 I)$ be the marginal distribution over $y$.
We construct a sequence of random measures $\mu_1, \mu_2, \cdots, $ by
sampling $y_i \simiid \mu$ and constructing the empirical measure
\begin{align*}
  \mu_m = \frac{1}{m} \sum_{i=1}^m \delta_{y_i}.
\end{align*}
In this case the optimization problem
\begin{equation} \label{eq:not-fitting-problem-hilbert-discretized}
	\begin{aligned}
		\minimize_{\est{\theta}(X, y_i) \in \real^{d}, 1 \leq i \leq m} & ~~ \int  \norm{\Sigma^{\frac{1}{2}} \prn{\est{\theta} - \prn{X^\top X + d \sigma^2 I}^{-1}X^\top y}}_2^2 d \mu_m  \nonumber \\
		\subjectto & ~~ \int \normtwo{X \est{\theta} - y}^2 d\mu_m \geq  \int \normtwo{X \est{\theta}(\wb{\rho}) - y}^2 d\mu_m 
	\end{aligned}
\end{equation}
is a finite dimensional optimization problem with (strongly convex)
quadratic objective and a single quadratic constraint. Then strong duality
obtains~\citep[Appendix~B.1]{BoydVa04}, so we can
write the stationary condition that for some $\lambda \geq 0$,
\begin{align*}
  \Sigma\prn{\est{\theta}(X,y_i) - \prn{X^\top X + d \sigma^2 I}^{-1}X^\top y_i} - \lambda X^\top (X\est{\theta}(X,y_i) -y_i) = 0
\end{align*}
simultaneously for $i = 1, \ldots, m$. Rewriting gives
\begin{align*}
  \prn{\Sigma - \lambda X^\top X} \est{\theta}(X, y_i) = \prn{\Sigma \prn{X^\top X + d \sigma^2 I}^{-1} - \lambda I} X^\top y_i,
  ~~ \mbox{for~} i = 1, \ldots , m.
\end{align*}
By an identical argument to that we use to prove
Lemma~\ref{lem:strong-duality} in Appendix~\ref{proof:strong-duality}, it
must be the case that
$\Sigma - \lambda X^\top X \succ 0$ and thus for each $i = 1, \ldots, m$,
\begin{align*}
	\est{\theta}(X, y_i) & = \prn{\Sigma - \lambda X^\top X}^{-1} \prn{\Sigma - \lambda X^\top X - \lambda d \sigma^2 I} \prn{X^\top X + d \sigma^2 I}^{-1} X^\top y_i \nonumber \\
	& = \prn{I - \lambda d \sigma^2 \prn{\Sigma - \lambda X^\top X}^{-1} } \prn{X^\top X + d \sigma^2 I}^{-1} X^\top y_i.
\end{align*}
By inspection,  this estimator is linear in $y$,
and for the choice $\lambda =
\frac{\wb{\rho}}{d}$ 
takes identical values at
$y_1, \ldots, y_m$ as $\est{\theta}(\wb{\rho})$. The constraints
of the problem~\eqref{eq:not-fitting-problem-hilbert-discretized}
are satisfied and the KKT conditions hold, so (an) optimal solution
is
$\est{\theta}(\wb{\rho})$.

For any $\est{\theta} \in
\Hsq(\epsilon)$, whenever the training errors satisfy
\begin{align*}
  \int \normtwo{X \est{\theta} - y}^2 d\mu_m \geq  \int \normtwo{X \est{\theta}(\wb{\rho}) - y}^2 d\mu_m,
\end{align*}
we must have
\begin{align}
  \int  \norm{\Sigma^{\frac{1}{2}} \prn{\est{\theta} -
      \E[\theta \mid y]}}_2^2 d \mu_m
  \geq \int \norm{\Sigma^{\frac{1}{2}} \prn{\est{\theta}(\wb{\rho})
      - \E[\theta \mid y]}}_2^2 d \mu_m.
  \label{eq:strong-duality-hilbert-mid-1}
\end{align}
By the law of large numbers, if $\est{\theta}$ is square integrable,
then with probability one
\begin{align*}
  \lim_{m \to \infty} \int \normtwo{X \est{\theta} - y}^2 d\mu_m = \int \normtwo{X \est{\theta} - y}^2 d\mu \geq \epsilon^2
  \stackrel{(\star)}{>}
  \int \normtwo{X \est{\theta}(\wb{\rho}) - y}^2 d\mu = \lim_{m \to \infty} \int \normtwo{X \est{\theta} - y}^2 d\mu_m,
\end{align*}
where inequality $(\star)$ holds by the assumption that $\wb{\rho} < \rho =
\rho(\epsilon)$, yielding the difference in training errors.  Thus
Eq.~\eqref{eq:strong-duality-hilbert-mid-1} holds eventually for all large
$m$. Again applying the law of large numbers and taking $m \to \infty$, we
establish the desired prediction error
gap~\eqref{eq:strong-duality-hilbert-claim}.

%% file: arxiv-main.bbl
\begin{thebibliography}{26}
\providecommand{\natexlab}[1]{#1}
\providecommand{\url}[1]{\texttt{#1}}
\expandafter\ifx\csname urlstyle\endcsname\relax
  \providecommand{\doi}[1]{doi: #1}\else
  \providecommand{\doi}{doi: \begingroup \urlstyle{rm}\Url}\fi

\bibitem[Anderson(1955)]{Anderson55}
T.~W. Anderson.
\newblock The integral of a symmetric unimodal function over a symmetric convex
  set and some probability inequalities.
\newblock \emph{Proceedings of the American Mathematical Society}, 6\penalty0
  (2):\penalty0 170--176, 1955.

\bibitem[Arora et~al.(2019{\natexlab{a}})Arora, Cohen, Hu, and
  Luo]{AroraCoHuLu19}
S.~Arora, N.~Cohen, W.~Hu, and Y.~Luo.
\newblock Implicit regularization in deep matrix factorization.
\newblock In \emph{Advances in Neural Information Processing Systems 32},
  2019{\natexlab{a}}.

\bibitem[Arora et~al.(2019{\natexlab{b}})Arora, Du, Hu, Li, and
  Wang]{AroraDuHuLiWang19}
S.~Arora, S.~Du, W.~Hu, Z.~Li, and R.~Wang.
\newblock Fine-grained analysis of optimization and generalization for
  overparameterized two-layer neural networks.
\newblock In \emph{Proceedings of the 36th International Conference on Machine
  Learning}, 2019{\natexlab{b}}.

\bibitem[Bai and Silverstein(2010)]{BaiSi10}
Z.~Bai and J.~W. Silverstein.
\newblock \emph{Spectral analysis of large dimensional random matrices},
  volume~20 of \emph{Springer Series in Statistics}.
\newblock Springer, 2010.

\bibitem[Bartlett et~al.(2020)Bartlett, Long, Lugosi, and
  Tsigler]{BartlettLoLuTs20}
P.~L. Bartlett, P.~M. Long, G.~Lugosi, and A.~Tsigler.
\newblock Benign overfitting in linear regression.
\newblock \emph{Proceedings of the National Academy of Sciences}, 117:\penalty0
  30063--30070, 2020.

\bibitem[Belkin et~al.(2018{\natexlab{a}})Belkin, Hsu, and Mitra]{BelkinHsMi18}
M.~Belkin, D.~Hsu, and P.~Mitra.
\newblock Overfitting or perfect fitting? {R}isk bounds for classification and
  regression rules that interpolate.
\newblock In \emph{Advances in Neural Information Processing Systems 31}, pages
  2300--2311. Curran Associates, Inc., 2018{\natexlab{a}}.

\bibitem[Belkin et~al.(2018{\natexlab{b}})Belkin, Ma, and Mandal]{BelkinMaMa18}
M.~Belkin, S.~Ma, and S.~Mandal.
\newblock To understand deep learning we need to understand kernel learning.
\newblock In \emph{Proceedings of the 35th International Conference on Machine
  Learning}, pages 541--549, 2018{\natexlab{b}}.

\bibitem[Belkin et~al.(2019)Belkin, Rakhlin, and Tsybakov]{BelkinRaTs19}
M.~Belkin, A.~Rakhlin, and A.~B. Tsybakov.
\newblock Does data interpolation contradict statistical optimality?
\newblock In \emph{Proceedings of the 22nd International Conference on
  Artificial Intelligence and Statistics}, pages 1611--1619, 2019.

\bibitem[Belkin et~al.(2020)Belkin, Hsu, and Xu]{BelkinHsXu20}
M.~Belkin, D.~Hsu, and J.~Xu.
\newblock Two models of double descent for weak features.
\newblock \emph{SIAM Journal on Mathematics of Data Science}, 2\penalty0
  (4):\penalty0 1167--1180, 2020.

\bibitem[Boyd and Vandenberghe(2004)]{BoydVa04}
S.~Boyd and L.~Vandenberghe.
\newblock \emph{Convex Optimization}.
\newblock Cambridge University Press, 2004.

\bibitem[Brown et~al.(2021)Brown, Bun, Feldman, Smith, and
  Talwar]{BrownBuFeSmTa21}
G.~Brown, M.~Bun, V.~Feldman, A.~Smith, and K.~Talwar.
\newblock When is memorization of irrelevant training data necessary for
  high-accuracy learning?
\newblock \emph{arXiv:2012.06421 [cs.LG]}, 2021.

\bibitem[Cover and Hart(1967)]{CoverHa67}
T.~M. Cover and P.~E. Hart.
\newblock Nearest neighbor pattern classification.
\newblock \emph{IEEE Transactions on Information Theory}, 13:\penalty0 21--27,
  1967.

\bibitem[Deng et~al.(2009)Deng, Dong, Socher, Li, Li, and
  Fei-Fei]{DengDoSoLiLiFe09}
J.~Deng, W.~Dong, R.~Socher, L.~Li, K.~Li, and L.~Fei-Fei.
\newblock Image{N}et: a large-scale hierarchical image database.
\newblock In \emph{Proceedings of the IEEE Conference on Computer Vision and
  Pattern Recognition}, pages 248--255, 2009.

\bibitem[Feldman(2020)]{Feldman20}
V.~Feldman.
\newblock Does learning require memorization? {A} short tale about a long tail.
\newblock In \emph{Proceedings of the Fifty-Second Annual ACM Symposium on the
  Theory of Computing}, pages 954--959, 2020.

\bibitem[Gunasekar et~al.(2017)Gunasekar, Woodworth, Bhojanapalli, Neyshabur,
  and Srebro]{GunasekarWoBhNeSr17}
S.~Gunasekar, B.~Woodworth, S.~Bhojanapalli, B.~Neyshabur, and N.~Srebro.
\newblock Implicit regularization in matrix factorization.
\newblock In \emph{Advances in Neural Information Processing Systems 30}, 2017.

\bibitem[Gunasekar et~al.(2018)Gunasekar, Lee, Soudry, and
  Srebro]{GunasekarLeSoSr18a}
S.~Gunasekar, J.~Lee, D.~Soudry, and N.~Srebro.
\newblock Characterizing implicit bias in terms of optimization geometry.
\newblock In \emph{Proceedings of the 35th International Conference on Machine
  Learning}, 2018.

\bibitem[Hastie et~al.(2019)Hastie, Montanari, Rosset, and
  Tibshirani]{HastieMoRoTi19}
T.~Hastie, A.~Montanari, S.~Rosset, and R.~Tibshirani.
\newblock Surprises in high-dimensional ridgeless linear least squares
  interpolation.
\newblock \emph{arXiv:1903.08560 [math.ST]}, 2019.

\bibitem[Ji and Telgarsky(2019)]{JiTe19}
Z.~Ji and M.~Telgarsky.
\newblock Gradient descent aligns the layers of deep linear networks.
\newblock In \emph{Proceedings of the Seventh International Conference on
  Learning Representations}, 2019.

\bibitem[Liang and Rakhlin(2020)]{LiangRa20}
T.~Liang and A.~Rakhlin.
\newblock Just interpolate: Kernel "ridgeless" regression can generalize.
\newblock \emph{Annals of Statistics}, 48:\penalty0 1329--1347, 2020.

\bibitem[Mei and Montanari(2021)]{MeiMo21}
S.~Mei and A.~Montanari.
\newblock The generalization error of random features regression: Precise
  asymptotics and double descent curve.
\newblock \emph{Communications on Pure and Applied Mathematics}, 2021.

\bibitem[Muthukumar et~al.(2019)Muthukumar, Vodrahalli, and
  Sahai]{MuthukumarVoSa19}
V.~Muthukumar, K.~Vodrahalli, and A.~Sahai.
\newblock Harmless interpolation of noisy data in regression.
\newblock In \emph{Proceedings of the IEEE International Symposium on
  Information Theory (ISIT)}, 2019.

\bibitem[Recht et~al.(2019)Recht, Roelofs, Schmidt, and Shankar]{RechtRoScSh19}
B.~Recht, R.~Roelofs, L.~Schmidt, and V.~Shankar.
\newblock Do {I}mage{N}et classifiers generalize to {I}mage{N}et?
\newblock In \emph{Proceedings of the 36th International Conference on Machine
  Learning}, 2019.

\bibitem[Silverstein(1995)]{Silverstein95}
J.~W. Silverstein.
\newblock Strong convergence of the empirical distribution of eigenvalues of
  large dimensional random matrices.
\newblock \emph{Journal of Multivariate Analysis}, 55\penalty0 (2):\penalty0
  331--339, 1995.

\bibitem[Soudry et~al.(2018)Soudry, Hoffer, Nacson, Gunasekar, and
  Srebro]{SoudryHoNaGuSr18}
D.~Soudry, E.~Hoffer, M.~S. Nacson, S.~Gunasekar, and N.~Srebro.
\newblock The implicit bias of gradient descent on separable data.
\newblock \emph{Journal of Machine Learning Research}, 19\penalty0
  (18):\penalty0 1--57, 2018.

\bibitem[Vapnik and Chervonenkis(1971)]{VapnikCh71}
V.~N. Vapnik and A.~Y. Chervonenkis.
\newblock On the uniform convergence of relative frequencies of events to their
  probabilities.
\newblock \emph{Theory of Probability and its Applications}, XVI\penalty0
  (2):\penalty0 264--280, 1971.

\bibitem[Zhang et~al.(2017)Zhang, Bengio, Hardt, Recht, and
  Vinyals]{ZhangBeHaReVi17}
C.~Zhang, S.~Bengio, M.~Hardt, B.~Recht, and O.~Vinyals.
\newblock Understanding deep learning requires rethinking generalization.
\newblock In \emph{Proceedings of the Fifth International Conference on
  Learning Representations}, 2017.

\end{thebibliography}
